\documentclass{article}

\PassOptionsToPackage{numbers, compress}{natbib}

\usepackage[main, final]{neurips_2026}

\usepackage[utf8]{inputenc} 
\usepackage[T1]{fontenc}    
\usepackage{hyperref}       
\usepackage{url}            
\usepackage{booktabs}       
\usepackage{amsfonts}       
\usepackage{nicefrac}       
\usepackage{microtype}      
\usepackage{xcolor}         
\usepackage{array}
\usepackage{booktabs}
\usepackage{makecell}
\usepackage{array}
\usepackage[table]{xcolor}
\usepackage{subcaption}
\usepackage{float}

\usepackage{algorithm}
\usepackage{algorithmic}
\usepackage{wrapfig}
\usepackage{siunitx}

\usepackage{amsmath}
\usepackage{amssymb}
\usepackage{mathtools}
\usepackage{amsthm}
\usepackage{enumitem}
\usepackage{bm}
\usepackage{bbm}
\usepackage[toc,page,header]{appendix}
\usepackage{minitoc}
\newtheorem{proposition}{Proposition}

\newcommand{\best}[1]{%
    \begingroup
    \setlength{\fboxsep}{1.5pt}      
    \setlength{\fboxrule}{0.1pt}    
    \fcolorbox{black}{gray!10}{\textbf{#1}}%
    \endgroup
}

\newcommand{\lightmidrule}{%
  
    \arrayrulecolor{gray!60}%
    \specialrule{0.3pt}{1pt}{1pt}%
    \arrayrulecolor{black}%
  
}

\usepackage{multirow}
\newcommand{\pms}[1]{\scalebox{0.75}{$\pm$ #1}}
\usepackage[capitalize,noabbrev]{cleveref}

\title{CVTA: Cross-Variable Temporal Attention for Multivariate Irregular Time Series Prediction}

\author{%
  Ankitkumar Joshi\\
  Department of Computer Science\\
  University of Pittsburgh\\
  Pittsburgh, PA 15260 \\
  \texttt{ahjoshi@pitt.edu} \\
  \And
  Milos Hauskrecht\\
  Department of Computer Science\\
  University of Pittsburgh\\
  Pittsburgh, PA 15260 \\
  \texttt{milos@pitt.edu} 
}

\doparttoc 
\faketableofcontents 

\begin{document}

\maketitle

\begin{abstract}
Irregularly sampled multivariate time series arise in many real-world domains such as healthcare, climate science, and large-scale monitoring systems. Characterized by asynchronous observations and heterogeneity across variables, such time series pose significant challenges for forecasting. In this work, we introduce Cross-Variable Temporal Attention (CVTA), a short-term, continuous-time forecasting framework that learns prediction functions from its compact, interpretable state summarizing recent temporal dynamics across variables. CVTA uses Variable-Conditioned Embeddings to utilize variable-specific latent spaces, and applies temporal attention to capture cross-variable dependencies from this state. We show that CVTA can be used either alone or integrated with existing state-of-the-art prediction models through a dynamic meta-decision model. Extensive experiments on multiple real-world benchmarks show that CVTA achieves strong short-term predictive accuracy, supports efficient online inference, improves existing prediction models when integrated with them, and learns attention patterns that align well with domain knowledge.
\end{abstract}

\section{Introduction}
Time series forecasting plays a critical role in a wide range of scientific and real-world applications, including healthcare, finance, and large-scale monitoring systems, where accurate predictions enable early event detection and informed decision-making under uncertainty. To better contextualize our problem setting, we classify time series along the two dimensions: marking and sampling. 
\begin{wraptable}{r}{0.48\linewidth}
\vspace{-1em}
\scriptsize
\centering
\caption{Characterization of time series types by marking and sampling.}
\label{tab:ts-types}
\begin{tabular}{l c >{\columncolor{gray!15}}c c}
\toprule
& \makecell{\textbf{Regular} \\ \textbf{Time Series}}
& \makecell{\textbf{Irregular} \\ \textbf{Time Series}}
& \makecell{\textbf{Event} \\ \textbf{Time Series}} \\
\midrule
\textbf{Marking}
& Continuous
& Continuous
& Discrete \\
\textbf{Sampling}
& Uniform
& Non-Uniform
& Non-Uniform \\
\bottomrule
\end{tabular}
\vspace{-1em}
\end{wraptable}
\textit{Marking} refers to the type of variable recorded over time; \textit{Sampling} describes how the observations are collected in time. \Cref{tab:ts-types} summarizes common time series types under this framework. For example, event time series consist of irregularly sampled discrete events, whereas irregular time series consist of asynchronously observed real-valued measurements (e.g., blood pressure, potassium lab test). Our work focuses on models for irregular multivariate time series prediction.

Real-world time series are often multivariate and irregularly sampled, with observations arriving asynchronously and variables governed by heterogeneous underlying processes.
To address irregular sampling, a variety of modeling approaches have been proposed, which can be broadly classified according to how they represent the underlying continuous-time dynamics. Differential equation–based approaches \cite{LatentODEPaper,GRUODEBayes,NeuralODE,CRUPaper,NeuralCDEPaper} model latent dynamics in continuous time using ordinary or stochastic differential equations. Autoregressive models for irregular time series \cite{GRUDPaper,NeuralHawkes,RMTPP,GRUwE} maintain a latent state whose evolution is governed by a fixed set of parameterized update equations. Attention-based models \cite{mTANDPaper,Raindrop,grafitiPaper,HyperIMTS} represent temporal dependencies using attention mechanisms in a learned embedding space.

\begin{wrapfigure}[24]{r}{0.45\linewidth}
    \vspace{-1.2em}
    \centering
    \includegraphics[width=\linewidth]{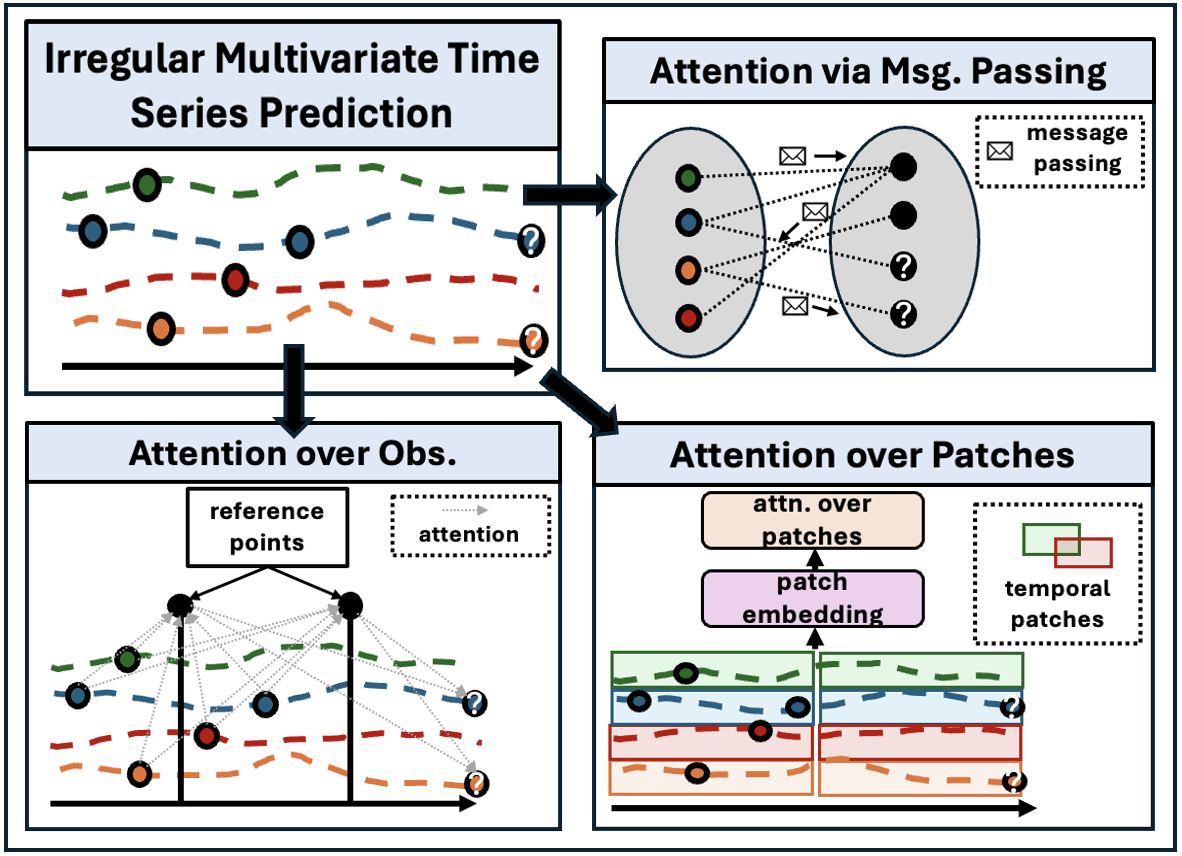}
    \caption{Attention-based architectures for irregular multivariate time series (\textbf{top left}) prediction: point-wise attention (\textbf{bottom left}) applies attention directly to individual observations; patch-based attention (\textbf{bottom right}) aggregates observations into temporal patches before applying attention; and structured message passing (\textbf{top right}) represents time points and variables as nodes in graphs or hypergraphs and learns attention via message passing.}
    \label{fig:attention_mechanisms}
\end{wrapfigure}

As shown in \cref{fig:attention_mechanisms}, existing attention-based forecasting methods model irregular time series by applying attention either at the level of individual observations \cite{mTANDPaper,TransformerHawkes,SAHP}, over temporal patches \cite{TPGNN,hiPatch}, or through structured message passing mechanisms \cite{Raindrop,grafitiPaper,HyperIMTS}. While expressive, most existing attention-based models face two challenges. First, models rely on embeddings that are shared across all observed variable types, implicitly enforcing a common notion of temporal relevance across heterogeneous variables. However, in many real-world multivariate time series, variables and their values may change at different scales and hence are often governed by distinct sampling processes. For example, in clinical monitoring, physiological signals may change rapidly and hence the elapsed time between consecutive measurements is often short (e.g., heart rate is measured continuously and recorded every few minutes in EHRs) while values of some laboratory tests evolve at slower pace and hence their sampling frequency is typically smaller (e.g., liver tests such as ALT are typically measured every 12 or 24 hours, lactate levels in critically ill unstable patients are reassessed every 2 to 4 hours, and more stable patients every 6 to 12 hours). 
Second, many existing models construct prediction functions that attend to full observation history. While this may be important in capturing possible long-range dependencies, the added benefit of full history analyses compared to the most recent set of observations and their dynamics \cite{JeongMinRecencyPaper} may be relatively small in short-term, especially in cases when sampling frequency of variables is already adjusted to their expected dynamics. Other drawbacks of full history analyses are increases in inference cost when they are performed online \cite{GRUwE} and interpretability of predictions.

The above observations motivate the development of models that explicitly account for heterogeneity in variable-specific dynamics and in sampling frequencies that reflect these dynamics. To address this, we propose \underline{C}ross-\underline{V}ariable \underline{T}emporal \underline{A}ttention (CVTA), 
a short-term continuous-time prediction model that learns from a set of recent observations and times elapsed since those observations. The key insight of our work is that this information provides sufficiently strong predictive signal to enable accurate short-term forecasting. The main contributions of our work are as follows:
\begin{itemize}[left=0pt, itemsep=2pt, parsep=0pt, topsep=0pt]
    \item We propose Cross-Variable Temporal Attention (CVTA), a novel short-term multivariate temporal attention mechanism for irregular time series prediction. Experiments on the next observation prediction task shows that CVTA exhibits strong performance against state-of-the-art (SOTA) prediction models, while offering efficient run time for both batched training and online inference.

    \item To the best of our knowledge, CVTA is the first neural models for irregular multivariate time series to combine an interpretable state with cross-variable temporal attention, jointly delivering the accuracy, efficiency, and interpretability required for real-world deployment in settings such as ICU monitoring. We verify the faithfulness of CVTA's interpretability through attention-erasure tests.
    \item We introduce a dynamic meta-decision framework that integrates CVTA 
    with existing irregular time series models, trained jointly in an end-to-end manner. This combination consistently improves base model performance across five benchmarks and four prediction horizons.
\end{itemize}

\section{Related Work}
This section reviews irregular multivariate time series prediction models and contrasts them to CVTA.

\textbf{Attention-based approaches.} Attention mechanisms have become widely used for irregular time series because they can flexibly model long-range dependencies and cross-variable interactions. Existing methods apply attention at different granularities: point-wise models such as mTAND \cite{mTANDPaper}, THP \cite{TransformerHawkes}, and SAHP \cite{SAHP} attend over individual observations or event arrivals; patch-based models such as T-PatchGNN \cite{TPGNN} and Hi-Patch \cite{hiPatch} aggregate observations into temporal segments before applying attention; and graph-based models such as Raindrop \cite{Raindrop}, GraFITi \cite{grafitiPaper}, and HyperIMTS \cite{HyperIMTS} use message passing over time-variable graph or hypergraph structures. 
As summarized in \Cref{fig:attention_mechanisms}, these approaches effectively capture long-range dependencies but often rely on all past observations, which can increase online inference cost and reduce interpretability. In contrast, CVTA uses a compact state representation based only on the most recent variable information. It also supports efficient and interpretable online predictions, making it well suited for real-world streaming applications.

\textbf{Alternative approaches.}
Differential equation--based models, including Latent ODE, ODE-RNN \cite{LatentODEPaper}, ContiFormer \cite{ContiFormer}, and Neural CDE \cite{NeuralCDEPaper}, model latent dynamics in continuous time, but often require costly numerical solvers. CRU and f-CRU \cite{CRUPaper} improve efficiency by using closed-form linear stochastic dynamics, while recurrent models with learned decay mechanisms \cite{GRUDPaper,NeuralHawkes,GRUwE} maintain compact latent states without needing to invoke external solvers. 
These models are often efficient for online inference, but their latent states and prediction functions are less directly interpretable, and their sequential state updates are harder to parallelize during batched training. In contrast, CVTA offers a compact interpretable state which together with the cross-variable temporal attention mechanism supports  efficient online inference, parallelism in batched training, and interpretable predictions.

\section{Cross-Variable Temporal Attention}
In this work, we develop CVTA for short-term prediction in multivariate irregular time series with $F > 1$ variables. At time $t$, CVTA summarizes the observed history $\mathcal{H}_t$ with a compact state reflecting most recent variable values and their trends, which we refer to as CVTA state. More specifically, it consists of recent values, elapsed times since last values were observed, and local slopes. Unlike existing attention models \cite{HyperIMTS,grafitiPaper,mTANDPaper} that rely on attention over the entire observed sequence, CVTA performs prediction by considering only the most recent information covered by the CVTA state. In the following, we first define this state, then introduce the cross-variable temporal attention mechanism, and finally describe how CVTA can be used as an auxiliary module in combination with other prediction models. 
\subsection{CVTA State} For each variable $f \in \{1,\dots,F\}$ at time $t \in \mathbb{R}$, CVTA maintains the following state:
\begin{equation}
\mathcal{S}_{t,f} =
\Big(
x_{t,f},\;
m_{t,f},\;
\delta_{t,f},\;
\dot{x}_{t,f},\;
g_{t,f}
\Big),
\end{equation}

where $x_{t,f} \in \mathbb{R}$ is the most recently observed value of variable $f$ up to time $t$;
the observation mask $m_{t,f} \in \{0,1\}$ indicates whether variable $f$ has been observed before $t$; $\delta_{t,f} \in \mathbb{R}_{\geq 0}$ is the time elapsed since the last observation of $f$; $\dot{x}_{t,f} \in \mathbb{R}$ is the local slope (rate of change) between the two most recent observations, computed as: 
\begin{equation}
    \dot{x}_{t,f} = \frac{o_{b, f} - o_{a,f}}{b - a}, \mathrm{where} \; a < b \leq t,
\end{equation}
where $o_{a,f}, o_{b,f}$ are observations of $f$ at times $a$ and $b$; $g_{t,f} \in \{0,1\}$ is a slope mask indicating whether $\dot{x}_{t,f}$ is defined (i.e., at least two prior observations exist). We refer to $\mathcal{S}_t=\{\mathcal{S}_{t,f}\}_{f=1}^{F}$ as the \emph{CVTA state}. It is maintained recursively (\crefrange{eq:update_x}{eq:slope_mask_update}) and can be viewed as an approximation of the Markov state of the process generating the observations.

\textbf{State Transition Model.} As shown in \cref{fig:markov_state}, the proposed state consists of three real-valued vectors: $\mathbf{x}_t,\bm{\delta}_t,\mathbf{\dot{x}}_t \in \mathbb{R}^F$ and two binary vectors: $\mathbf{m}_t, \mathbf{g}_t \in \{0,1\}^F$. We now describe how the CVTA state components are initialized and updated sequentially upon arrival of a new observation. We initialize the state components as:
\[
(\mathbf{x}_0,\;\mathbf{m}_0,\;\bm{\delta}_0,\;\mathbf{ \dot{x}}_0,\; \mathbf{g}_0)=\mathbf{0}.
\]
Given the previous state $\mathcal{S}_{t-1}$, consider a new observation $\textbf{o}_t\in\mathbb{R}^F$, and binary observation indicator vector $\textbf{i}_t \in \{0,1\}^F$arriving after time interval $\Delta \tau \in \mathbb{R}_{>0}$ \footnote{In equations \eqref{eq:update_delta} and \eqref{eq:update_slope}, the scalar time increment \(\Delta \tau\) is broadcast to match the dimension of the vector \(\bm{\delta}_{t-1}\). This is equivalent to \(\bm{\Delta\tau} = \Delta \tau \cdot \mathbbm{1}\), where $\mathbbm{1} \in \mathbb{R}^F$, representing a uniform time span applied across all variables.}. Previous state $S_{t-1}$ is updated using \crefrange{eq:update_x}{eq:slope_mask_update} where, $\mathrm{max}$ and $\odot$ denote element-wise maximum and multiplication operations respectively. The update in \cref{eq:update_x} stores the most recent observations by replacing only those entries corresponding to newly observed variables. The mask update in \cref{eq:update_mask} records whether each variable has ever been observed, which is necessary to distinguish true zero values from missing data. As shown in \cref{eq:update_delta}, the elapsed time is reset for newly observed variables and incremented by $\bm{\Delta \tau}$ otherwise. Since the slope calculation requires at least two observations, the intermediate gate $\mathbf{z}_t \in \{0,1\}^F$ in \cref{eq:slope_mask} is computed which activates only when both current and prior observations are available. Under this gating, \cref{eq:update_slope} updates the slope where possible and carries forward the previous estimate otherwise. Finally, similar to $\mathbf{m}_t$, the slope mask $\mathbf{g}_t$ in \cref{eq:slope_mask_update} tracks if a valid slope has been computed for each variable.

\begin{figure*}[t]
    \centering
    \begin{minipage}[c]{0.47\textwidth}
        \centering
        \includegraphics[width=\linewidth]{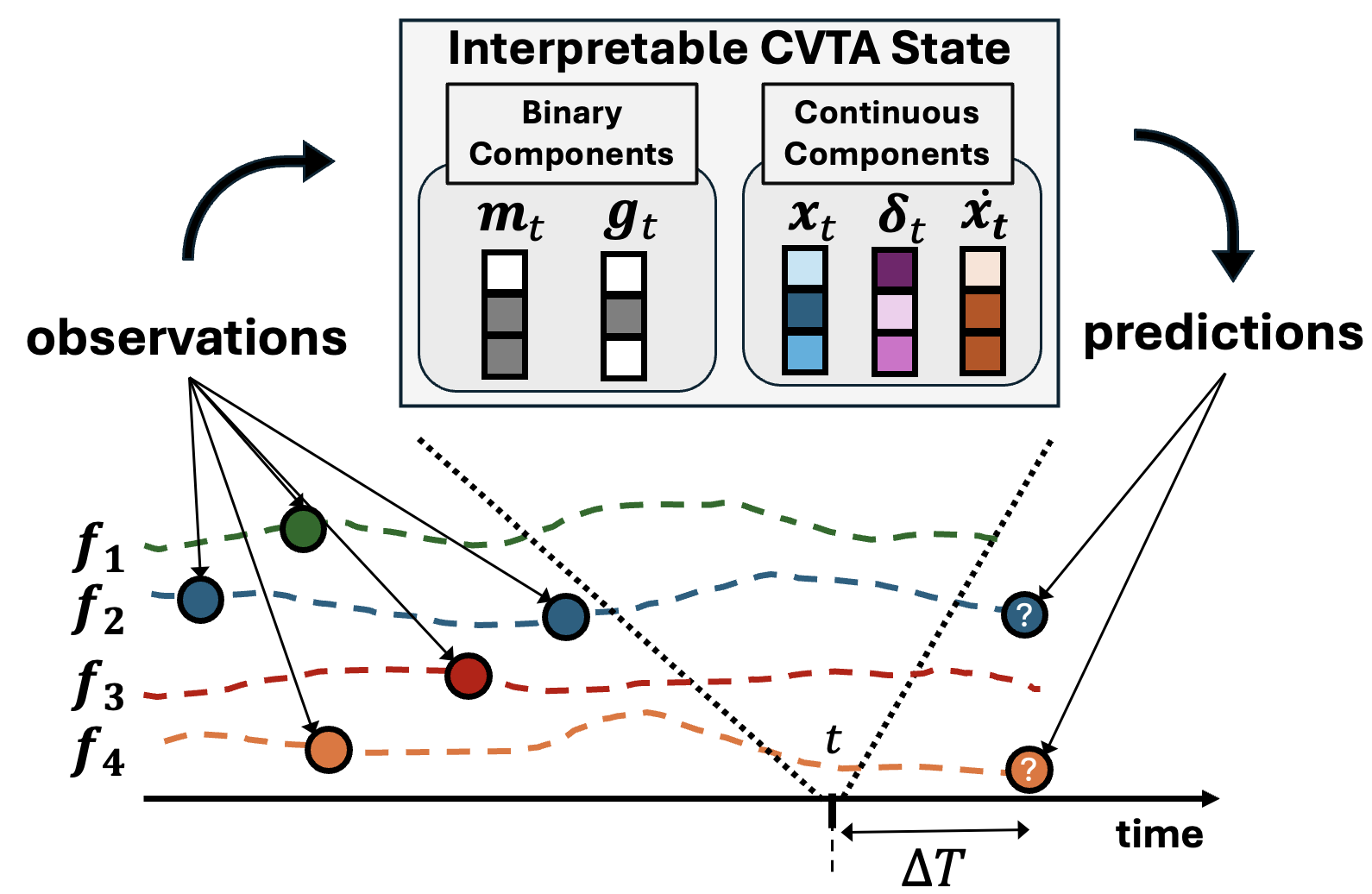}
        \caption{
CVTA state representation.
        }
        \label{fig:markov_state}
    \end{minipage}
    \hfill
    \begin{minipage}[c]{0.52\textwidth}
        \small
        \begin{align}
        \mathbf{x}_{t} &= \mathbf{i}_{t} \odot \mathbf{o}_{t} + (\mathbbm{1} - \mathbf{i}_{t}) \odot \mathbf{x}_{t-1}, \label{eq:update_x} \\
        \mathbf{m}_{t} &= \max(\mathbf{i}_t, \, \mathbf{m}_{t-1}), \label{eq:update_mask} \\
        \bm{\delta}_{t} &= \mathbf{i}_{t} \odot \mathbf{0} + (\mathbbm{1} - \mathbf{i}_{t}) \odot (\bm{\delta}_{t-1} + \bm{\Delta \tau}), \label{eq:update_delta} \\
        \mathbf{z}_t &= \mathbf{i}_{t} \odot \mathbf{m}_{t-1}, \label{eq:slope_mask}\\
        \mathbf{\dot{x}}_{t} &= \mathbf{z}_t \odot \frac{\mathbf{o}_{t} - \mathbf{x}_{t-1}}{\bm{\delta}_{t-1} + \bm{\Delta\tau}} + (\mathbbm{1} - \mathbf{z}_t) \odot \mathbf{\dot{x}}_{t-1}, \label{eq:update_slope}\\
        \mathbf{g}_{t} &= \max(\mathbf{z}_t, \, \mathbf{g}_{t-1}).
        \label{eq:slope_mask_update}
        \end{align}
        \captionof*{figure}{Eqs.~\eqref{eq:update_x}--\eqref{eq:slope_mask_update}: CVTA State Update Equations.}
    \end{minipage}
    \vspace{-1em}
\end{figure*}

\subsection{CVTA: Cross-Variable Temporal Attention}
The CVTA state defines a fixed-dimensional representation of the multivariate process, invariant to sequence length. Building on this state, we learn a continuous-time prediction function through a Cross-Variable Temporal Attention (CVTA) mechanism. This proceeds in two stages: (i) projecting the state components into a learned embedding space, and (ii) modeling dependencies across variables with the help of attention over these embeddings.

\textbf{Variable-Conditioned Embeddings.}
Real-world multivariate time series exhibit substantial variable-level heterogeneity, including differences in statistical properties, semantics, noise characteristics, and temporal dynamics \footnote{We provide a detailed motivation for modeling variable-level heterogeneity in multivariate time series in the Appendix~\ref{appendix:motivation_variable_cond_emb}.}. A shared embedding across variables can therefore be limiting: when heterogeneous variables share the same scalar state values, shared embeddings may collapse their variable identity (see \cref{prop:shared_emb_collapse}). To address this, we introduce Variable-Conditioned Embeddings (VCEs), which parameterize the representation of each variable separately rather than enforcing a shared embedding space. More concretely, for each variable \(f\), we define a set of embedding operators  
\[
\phi^{(f)}_{\bullet} : \mathbb{R} \to \mathbb{R}^D, \qquad \bullet \in \{x, m, \delta_r, \dot{x}, g, \delta_s\},
\]  
where, $\delta_r$ and $\delta_s$ define two distinct embeddings of the same elapsed time ($\delta$), supporting the last value and slope attention branches. The $\phi_{\bullet}^{(f)}$ is implemented as an affine mapping  
\[
\phi^{(f)}_{\bullet}(z) = \mathbf{w}^{(f)}_{\bullet} z + \mathbf{b}^{(f)}_{\bullet},
\] 
where \(\mathbf{w}^{(f)}_{\bullet}, \mathbf{b}^{(f)}_{\bullet} \in \mathbb{R}^{D}\) are learned parameters specific to both the variable \(f\) and the CVTA state component \(\bullet\), and \(D\) is the embedding dimension. Given the CVTA state \(\mathcal{S}_{t,f}\) for variable $f$ at time \(t\) and a prediction horizon \(\Delta T > 0\), we first adjust the elapsed time to \(\delta'_{t,f} = \delta_{t,f} + \Delta T\), aligning the state with the target prediction time. To preserve potential periodicity in temporal dynamics, we embed \(\delta'_{t,f}\) using a variable-conditioned affine transformation followed by a sinusoidal nonlinearity. We learn two separate temporal embeddings (one for each attention branch introduced later) along with embeddings of the remaining state components:
\[
\begin{aligned}
\mathbf{e}^{(f)}_{\delta_r} &= \sin\!\left( \phi^{(f)}_{\delta_r}(\delta'_{t,f}) \right), &
\mathbf{e}^{(f)}_{\delta_s} &= \sin\!\left( \phi^{(f)}_{\delta_s}(\delta'_{t,f}) \right), &
\mathbf{e}^{(f)}_{m}        &= \phi^{(f)}_{m}(m_{t,f}), \\
\mathbf{e}^{(f)}_{g}        &= \phi^{(f)}_{g}(g_{t,f}), &
\mathbf{e}^{(f)}_{x}        &= \phi^{(f)}_{x}(x_{t,f}), &
\mathbf{e}^{(f)}_{\dot{x}} &= \phi^{(f)}_{\dot{x}}(\dot{x}_{t,f}).
\end{aligned}
\]

\begin{figure*}[t]
    \centering
    \begin{subfigure}{0.68\linewidth}
        \centering
        \includegraphics[width=\linewidth]{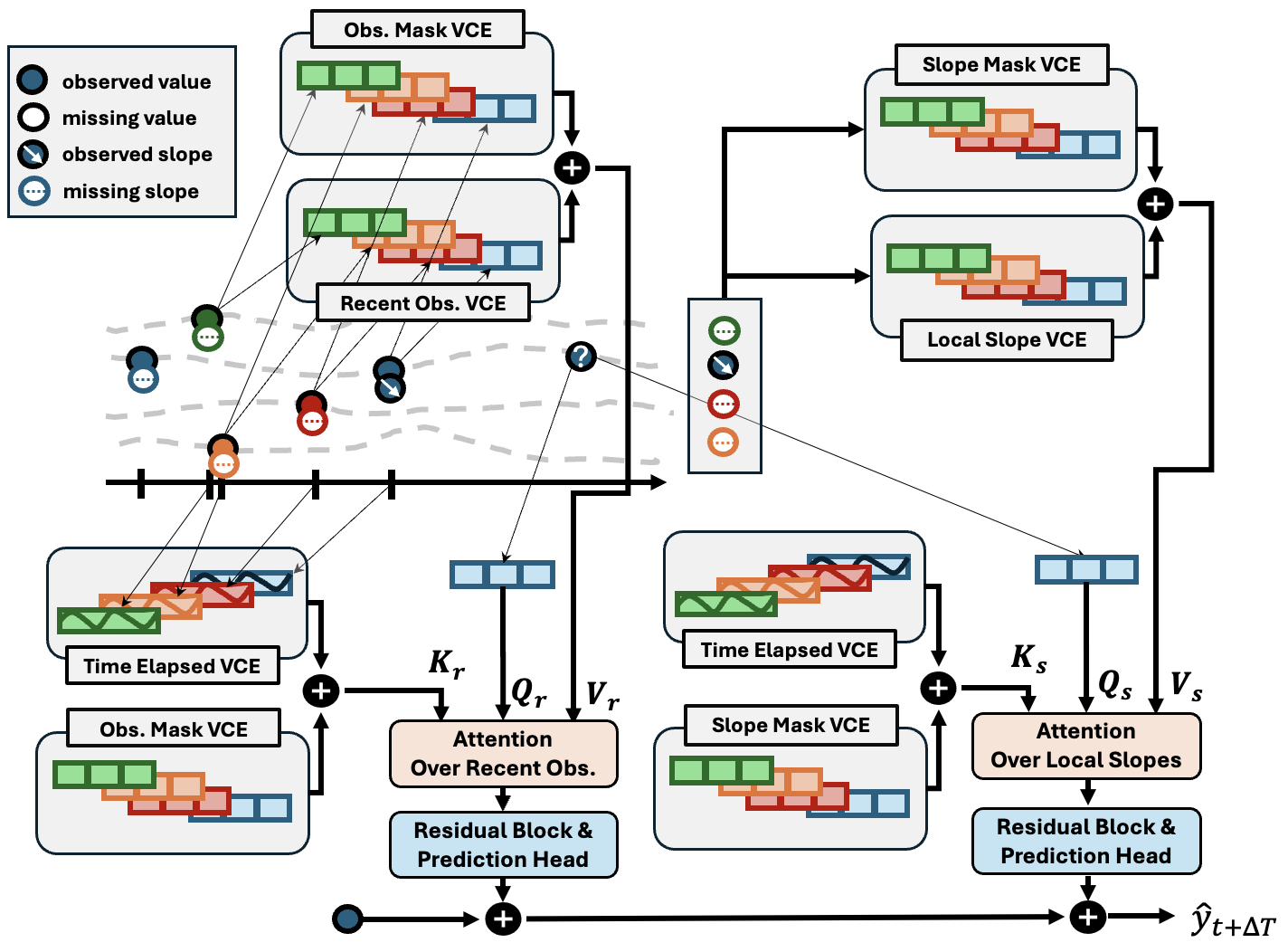}
        \caption{CVTA model architecture.}
        \label{fig:cvta}
    \end{subfigure}
    \hfill
    \begin{subfigure}{0.31\linewidth}
        \centering
        \includegraphics[width=\linewidth]{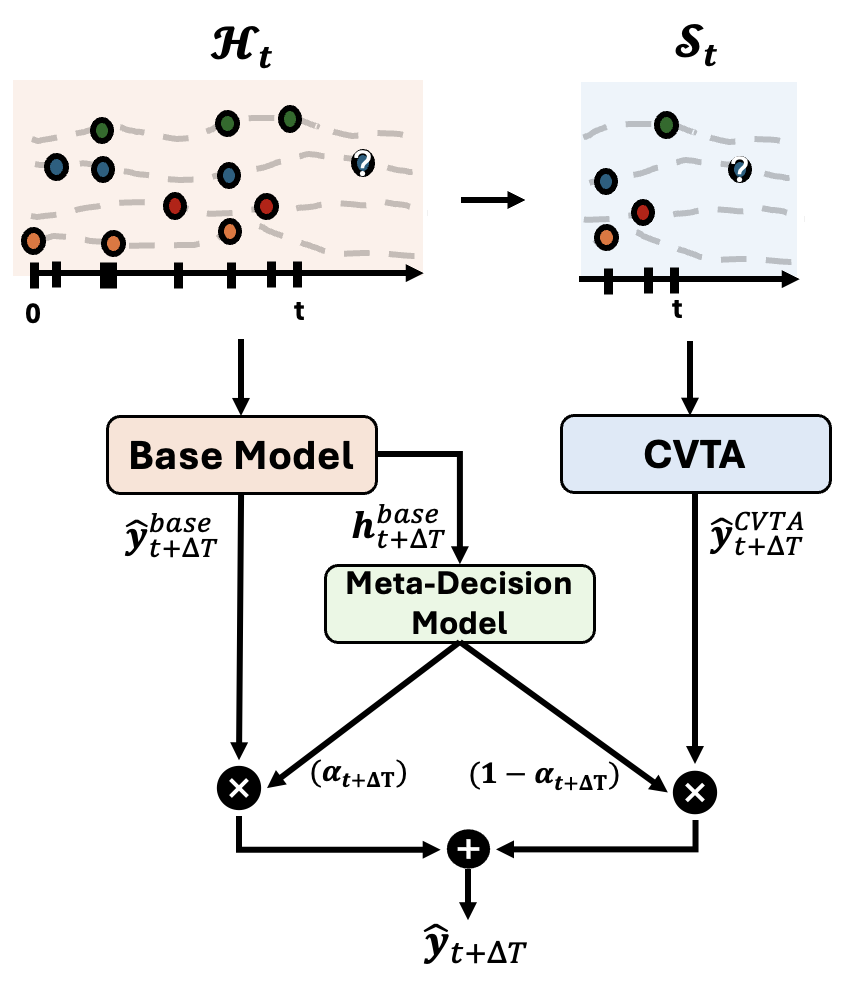}
        \caption{CVTA as an auxiliary prediction module.}
        \label{fig:cvta_as_aux_module}
    \end{subfigure}
    \caption{
    Cross-Variable Temporal Attention (CVTA) architectures.
    (\textbf{Left}) CVTA model performs attention over recent observations and local dynamics using Variable-Conditioned Embeddings (VCEs), enabling robust forecasting under irregular sampling and missing data. (\textbf{Right}) CVTA integrated as an auxiliary module with a base model and dynamic meta-decision.
    }
    \label{fig:cvta_combined}
    \vspace{-1em}
\end{figure*}

\textbf{Cross-Variable Temporal Attention.}
We model cross-variable dependencies with the help of the attention \cite{AttentionPaper} operation over the variable-conditioned embeddings. To preserve interpretability in continuous-time predictions, we decompose the attention into two additive branches: one focused on recent observed values, and another on local slope dynamics.

\textbf{Attention over Recent Observations.}
This attention branch captures the dependencies between variables based on their most recent observations and their time elapsed. For each variable~\(f\), queries are derived from a learned variable embedding, keys from horizon-adjusted elapsed time and mask embeddings, and values from the recent observation and mask embeddings:
\begin{equation}
\mathbf{q}_r^{(f)} = \mathbf{E}_r^\top \mathbf{c}^{(f)}, \quad
\mathbf{k}_r^{(f)} = \mathbf{e}^{(f)}_{\delta_r} + \mathbf{e}^{(f)}_{m}, \quad
\mathbf{v}_r^{(f)} = \mathbf{e}^{(f)}_{x} + \mathbf{e}^{(f)}_{m}.
\label{eq:attn_over_rec_obs}
\end{equation}
where, \(\mathbf{c}^{(f)} \in \{0,1\}^{F}\) is a one-hot variable indicator, \(\mathbf{E}_r \in \mathbb{R}^{F \times D}\) is a learnable embedding matrix and $\mathbf{q}_r^{(f)}, \mathbf{k}_r^{(f)}, \mathbf{v}_r^{(f)} \in \mathbb{R}^D$. Incorporating the mask embedding \(\mathbf{e}^{(f)}_{m}\) into both the key and value allows the model to distinguish between missing observations and observed zero values. Stacking the per‑variable vectors into matrices $\mathbf{Q}_r, \mathbf{K}_r, \mathbf{V}_r \in \mathbb{R}^{F \times D}$, we compute the latent representation at query points using scaled dot‑product attention:
\[
\mathbf{H}_r = \mathrm{Attention}(\mathbf{Q}_r, \mathbf{K}_r, \mathbf{V}_r).
\]

\textbf{Attention over Local Slope.} 
Beyond recent values, local temporal trends provide critical signals for short-term future predictions. To capture cross-variable dependencies in these dynamics, we introduce a second attention branch over slope embeddings. Similar to the recent-observation branch, for each variable \(f\) we define:
\begin{equation}
\mathbf{q}_s^{(f)} = \mathbf{E}_s^\top \mathbf{c}^{(f)},\quad
\mathbf{k}_s^{(f)} = \mathbf{e}^{(f)}_{\delta_s} + \mathbf{e}^{(f)}_{g},\quad 
\mathbf{v}_s^{(f)} = \mathbf{e}^{(f)}_{\dot{x}} + \mathbf{e}^{(f)}_{g},
\end{equation}
where \(\mathbf{E}_s \in \mathbb{R}^{F \times D}\) is a learnable query embedding for slopes, \(\mathbf{e}^{(f)}_{\delta_s}\) is a separate elapsed-time embedding for slopes, \(\mathbf{e}^{(f)}_{\dot{x}}\) is the embedded local slope, and \(\mathbf{e}^{(f)}_{g}\) is the slope mask embedding, with all branch quantities having the same dimensions as in the recent-observation branch. Similar to \cref{eq:attn_over_rec_obs}, the key and the value explicitly combine their embeddings with the slope mask embedding, allowing the model to distinguish between defined and undefined trends (e.g., when fewer than two observations are made). We obtain the slope-aware latent representation:
\begin{equation}
    \mathbf{H}_s = \mathrm{Attention}(\mathbf{Q}_s, \mathbf{K}_s, \mathbf{V}_s).
\end{equation}
\textbf{Residual Blocks.} The output of the two attention branches $\mathbf{H}_{\bullet} \in \mathbb{R}^{F \times D}, \bullet \in \{r, s\}$, is further refined using a position-wise feed-forward network (FFN) with a residual connection and layer normalization:
\begin{equation}
\mathbf{\tilde{H}}_{\bullet} = \mathrm{LayerNorm}\!\left( \mathbf{H}_{\bullet} + \mathrm{FFN}(\mathbf{H}_{\bullet}) \right).
\end{equation}
Here, $\mathrm{FFN}(\cdot)$ operation consists of a two-layer feed-forward network with $\mathrm{ReLU}$ intermediate activation. Both $\mathrm{FFN}(\cdot)$ and $\mathrm{LayerNorm}(\cdot)$ are learned independently for the two attention branches.

\textbf{Prediction Heads.} For each attention branch, we map the residual block output $\mathbf{\tilde{H}}_{\bullet}\in \mathbb{R}^{F \times D}$ to a prediction through a linear projection:
\begin{equation}
\hat{\mathbf{y}}_{\bullet} = \tilde{\mathbf{H}}_{\bullet}\mathbf{W}_{\bullet} + b_{\bullet},
\end{equation}
where each attention branch uses its own set of projection parameters $\{\mathbf{W}_{\bullet} \in \mathbb{R}^{D \times 1}, b_{\bullet} \in \mathbb{R}\}$, yielding $\hat{\mathbf{y}}_{\bullet} \in \mathbb{R}^F$.

\textbf{Final Prediction.} The final prediction at time $\Delta T \in \mathbb{R}_{>0}$ in the future having seen observations until time $t$, is obtained by adding the estimates derived from recent observation ($\mathbf{\hat{y}_r}$) and local slope prediction heads ($\mathbf{\hat{y}_s}$), giving $\hat{\mathbf{y}}_{t+\Delta T} \in \mathbb{R}^F$:
\begin{equation}
\label{eq:final_prediction}
\mathbf{\hat{y}}_{t+\Delta T} = \mathbf{x}_t + \mathbf{\hat{y}}_{r} + \mathbf{\hat{y}}_{s}.
\end{equation}
\textbf{Summary of CVTA.} CVTA is a short-term continuous-time prediction model for irregular multivariate time series. It approximates local temporal dynamics with a compact CVTA state and as illustrated in \cref{fig:cvta}, it uses a dual-branch attention mechanism to enable interpretable predictions.

\subsection{Integration of CVTA with other prediction models} 
\label{section:cvta_as_aux_pred_head} 
CVTA relies on recent temporal dynamics to formulate short-term predictions. For tasks requiring longer-term context, CVTA can be integrated as an auxiliary module into existing architectures with the goal of improving prediction by better utilizing recent observations. In this setup, the resulting combined model consists of two complementary prediction heads, along with a dynamic meta-decision model that adaptively merges their predictions based on the latent representation extracted from the base model.

\textbf{Meta-Decision Model.} The meta-decision weight $\alpha_{t+\Delta T}$ is computed using a feed‑forward network (FFN) applied to the base model’s latent representation, followed by a sigmoid activation $\sigma$ to constrain the weight to be in the interval $[0,1]$:
\begin{equation}
    \alpha_{t+\Delta T} = \displaystyle \sigma\Bigl(\,\mathrm{FFN}\bigl(\mathbf{h}_{t+\Delta T}^{{\scriptscriptstyle \mathrm{base    }}}\bigr)\,\Bigr).
\end{equation}
As shown in \cref{fig:cvta_as_aux_module}, the final prediction is a convex combination of the base model and CVTA predictions, weighted by the learned meta-decision output:
\begin{equation}
    \mathbf{\hat{y}}_{t+\Delta T} = \alpha_{t+\Delta T} \cdot \mathbf{\hat{y}}_{t+\Delta T}^{{\scriptscriptstyle \mathrm{base}}} + (1 - \alpha_{t+\Delta T}) \cdot \mathbf{\hat{y}}_{t+\Delta T}^{{\scriptscriptstyle \mathrm{CVTA}}}.
\end{equation}
To demonstrate the generality of CVTA as an auxiliary module, we integrate it with three top-performing benchmark models: GraFITi \cite{grafitiPaper}, GRUwE \cite{GRUwE}, and HyperIMTS \cite{HyperIMTS}. Additional implementation details are provided in Appendix~\ref{section:cvta_base_model_integration}.

\section{Empirical Evaluation}
\begin{wraptable}{r}{0.52\linewidth}
\centering
\label{tab:datasets}
\tiny
\vspace{-7.1em}
\caption{Dataset statistics.}
\setlength{\tabcolsep}{3pt}
\begin{tabular}{l@{\hspace{1pt}}|ccccc}
\toprule 
&
\textbf{\rotatebox{90}{USHCN}} &
\textbf{\rotatebox{90}{Physionet}} &
\textbf{\rotatebox{90}{MIMIC-III}} &
\textbf{\rotatebox{90}{MIMIC-IV}} &
\textbf{\rotatebox{90}{MIMIC-HiRes}} \\
\midrule
\textbf{Missingness type} & MCAR & NMAR & NMAR & NMAR & NMAR \\
\textbf{\# Vars.} & 5 & 37 & 96 & 100 & 393 \\
\textbf{Mean len} & 730.0  & 72.16 & 30.23 &  154.34 & 244.30 \\
\textbf{Max len} & 730 & 185 & 96 & 920 & 4299 \\
\textbf{Interarrival Variability} & 0.894 & 0.609  & 1.385 & 1.609 & 1.015 \\
\textbf{Missingness (\%)} & 22.36 & 83.85 & 95.59 & 97.83 & 98.74  \\
\textbf{Time Resolution} & 1d & 6m & 30m & 1m & 1s \\
\textbf{Context time span} & 365d & 24h &  36h & 36h & 48h \\
\textbf{Avg. Pred. horizon} & 2.85d & 1.19h & 6.11h & 4.33h  & 1.70m \\
\textbf{Std. Pred. horizon} & 0.35d & 2.10h & 5.73h & 6.91h & 1.58m\\
\bottomrule
\end{tabular}
\label{tab:dataset_descriptive_stats}
\vspace{-3.0em}
\end{wraptable}
Following \cite{GRUwE}, we evaluate CVTA on the next observation prediction task across multiple irregular time series benchmarks.

\textbf{Baselines.} We compare our proposed model CVTA with existing recurrent models for irregular time series including: CRU, f-CRU \cite{CRUPaper}, RKN-$\Delta_t$ \cite{RKNPaper}, GRU-$\Delta_t$ \cite{GRUPaper}, GRU-D \cite{GRUDPaper}, GRUwE \cite{GRUwE}; differential equation based approaches includes Latent ODE, ODE-RNN \cite{LatentODEPaper}, ContiFormer \cite{ContiFormer}; attention approaches including mTAND \cite{mTANDPaper}, T-PatchGNN \cite{TPGNN}, GraFITi \cite{grafitiPaper} and HyperIMTS \cite{HyperIMTS}.

\textbf{Datasets.} We benchmark our method against state-of-the-art approaches using five irregular time series datasets: USHCN \cite{USHCNDataset}, Physionet \cite{PhysionetDataset}, MIMIC-III \cite{MIMICIIIDataset}, MIMIC-HiRes \cite{MIMICIIIDataset} and MIMIC-IV \cite{MIMIC_IV}. Dataset statistics are reported in \Cref{tab:dataset_descriptive_stats} and how they are extracted is described in Appendix~\ref{appendix:dataset_stats}.

\textbf{Model Training.}
Following \cite{CRUPaper,GRUwE}, models are trained to reconstruct the full observation sequence from the historical observations using a masked MSE loss. For $N$ multivariate sequences, let $T_n$ denote the length of the $n$-th sequence, and let $i^{(n)}_{t,f} \in \{0,1\}$ indicate whether variable $f$ is observed at time step $t$ of the $n$-th sequence. The loss is given by:
\begin{equation}
\mathcal{L}
=
\frac{1}{
\sum_{n=1}^{N}\sum_{t=1}^{T_n}\sum_{f=1}^{F} i^{(n)}_{t,f}
}
\sum_{n=1}^{N}\sum_{t=1}^{T_n}\sum_{f=1}^{F}
i^{(n)}_{t,f}
\left(
\hat{y}^{(n)}_{t,f} - y^{(n)}_{t,f}
\right)^2 .
\label{eq:masked_mse_loss}
\end{equation}

\textbf{Hyperparameter Search.} We conduct a grid search over key hyperparameters to identify the optimal configuration. We provide the search space and training settings in Appendix~\ref{appendix:hyperparameter_search}.

\begin{table*}[]
\tiny
\centering
\caption{Next observation prediction results. We report the mean and standard deviation of MSE and MAE (lower is better), over five independent random seeds. "OOM" refers to out-of-memory GPU error encountered during training. For better readability, we annotate the\best{best} and the \underline{second-best} models.}
\setlength{\tabcolsep}{1.6pt}
\begin{tabular}{l|cc|cc|cc|cc|cc}
\toprule
\multirow{2}{*}{\textbf{Model}}
& \multicolumn{2}{c}{\textbf{USHCN}}
& \multicolumn{2}{c}{\textbf{Physionet}}
& \multicolumn{2}{c}{\textbf{MIMIC-III}}
& \multicolumn{2}{c}{\textbf{MIMIC-IV}}
& \multicolumn{2}{c}{\textbf{MIMIC-HiRes}} \\
\cmidrule(lr){2-3} \cmidrule(lr){4-5} \cmidrule(lr){6-7} \cmidrule(lr){8-9} \cmidrule(lr){10-11}
&  MSE$\scalebox{0.75}{$\times 10^{-2}$}$ & MAE $\scalebox{0.75}{$\times 10^{-2}$}$
& MSE $\scalebox{0.75}{$\times 10^{-2}$}$ & MAE $\scalebox{0.75}{$\times 10^{-2}$}$
& MSE $\scalebox{0.75}{$\times 10^{-1}$}$ & MAE $\scalebox{0.75}{$\times 10^{-1}$}$
& MSE $\scalebox{0.75}{$\times 10^{-1}$}$ & MAE $\scalebox{0.75}{$\times 10^{-1}$}$
& MSE $\scalebox{0.75}{$\times 10^{-2}$}$ & MAE $\scalebox{0.75}{$\times 10^{-2}$}$ \\
\midrule
\textbf{CRU}
& $0.030\pms{0.019}$ & $0.519\pms{0.166}$
& $0.807\pms{0.035}$ & $5.233\pms{0.242}$
& $6.913\pms{\text{0.040}}$ & $5.797\pms{\text{0.103}}$
& $4.013\pms{\text{0.022}}$ & $4.128\pms{\text{0.141}}$
& $1.236\pms{0.035}$ & $6.735\pms{0.139}$ \\

\textbf{f-CRU}
& $0.020\pms{0.007}$ & $0.455\pms{0.077}$
& $1.095\pms{0.069}$ & $5.295\pms{0.132}$
& $7.021\pms{\text{0.030}}$ & $6.134\pms{\text{0.052}}$
& $4.142\pms{\text{0.021}}$ & $4.290\pms{\text{0.076}}$
& $1.191\pms{0.043}$ & $6.535\pms{0.214}$ \\

\textbf{RKN-$\Delta_t$}
& $0.015\pms{0.016}$ & $0.367\pms{0.178}$
& $0.680\pms{0.042}$ & $4.854\pms{0.197}$
& $7.784\pms{0.086}$ & $6.102\pms{0.038}$
& $5.707\pms{\text{0.031}}$ & $5.961\pms{\text{0.029}}$
& $1.292\pms{0.042}$ & $6.820\pms{0.081}$ \\

\textbf{GRU-$\Delta_t$}
& $0.035\pms{0.001}$ & $0.717\pms{0.018}$
& $0.449\pms{0.018}$ & $4.160\pms{0.154}$
& $5.731\pms{0.009}$ & $4.946\pms{0.021}$
& $4.359\pms{\text{0.005}}$ & $4.729\pms{\text{0.008}}$
& $1.414\pms{0.051}$ & $7.226\pms{0.086}$ \\

\textbf{GRU-D}
& $0.015\pms{0.009}$ & $0.386\pms{0.110}$
& $0.672\pms{0.026}$ & $5.459\pms{0.094}$
& $5.906\pms{\text{0.028}}$ & $5.362\pms{\text{0.046}}$
& $4.741\pms{\text{0.011}}$ & $4.867\pms{\text{0.038}}$
& $0.984\pms{0.028}$ & $5.853\pms{0.102}$ \\

\textbf{GRUwE}
& $0.015\pms{0.013}$ & $0.245\pms{0.187}$
& $0.233\pms{0.006}$ & $2.834\pms{0.035}$
& $4.016\pms{\text{0.006}}$ & $3.913\pms{\text{0.023}}$
& $2.113\pms{\text{0.009}}$ & $2.970\pms{\text{0.060}}$
& $\underline{0.816}\pms{0.026}$ & $\underline{5.278}\pms{0.036}$ \\

\lightmidrule

\textbf{Latent ODE}
& $0.007\pms{0.004}$ & $0.127\pms{0.031}$
& $0.676\pms{0.005}$ & $5.302\pms{0.001}$
& $6.598\pms{0.049}$ & $5.539\pms{0.049}$
& $4.981\pms{\text{0.010}}$ & $5.114\pms{\text{0.029}}$
& $1.209\pms{0.018}$ & $6.490\pms{0.049}$ \\

\textbf{ContiFormer}
& $\underline{0.005}\pms{0.002}$ & $\underline{0.125}\pms{0.016}$
& $0.479\pms{0.024}$ & $4.232\pms{0.001}$
& $7.140\pms{\text{0.029}}$ & $5.896\pms{\text{0.011}}$
& $5.612\pms{\text{0.018}}$ & $5.838\pms{\text{0.054}}$
& $1.348\pms{0.093}$ & $6.941\pms{0.399}$ \\

\textbf{ODE-RNN}
& $0.019\pms{0.017}$ & $0.220\pms{0.116}$
& $0.770\pms{0.042}$ & $5.519\pms{0.273}$
& $6.806\pms{\text{0.013}}$ & $5.754\pms{\text{0.118}}$
& $5.544\pms{\text{0.026}}$ & $5.725\pms{\text{0.072}}$
& $1.429\pms{0.046}$ & $7.251\pms{0.262}$ \\

\lightmidrule

\textbf{mTAND}
& $0.007\pms{0.004}$ & $0.194\pms{0.069}$
& $0.330\pms{0.015}$ & $3.411\pms{0.130}$
& $5.804\pms{\text{0.022}}$ & $5.054\pms{\text{0.031}}$
& $3.773\pms{\text{0.013}}$ & $4.296\pms{\text{0.126}}$
& $1.260\pms{0.018}$ & $6.885\pms{0.082}$ \\

\textbf{T-PatchGNN}
& $0.065\pms{0.031}$ & $0.909\pms{0.405}$
& $0.338\pms{0.032}$ & $3.259\pms{0.168}$
& $4.655\pms{\text{0.031}}$ & $4.219\pms{\text{0.164}}$
& $2.634\pms{\text{0.021}}$ & $3.476\pms{\text{0.036}}$
& $1.226\pms{0.010}$ & $6.689\pms{0.151}$ \\

\textbf{GraFITi}
& $0.074\pms{0.015}$ & $0.673\pms{0.042}$
& $\underline{0.232}\pms{\text{0.011}}$ & $\underline{2.684}\pms{\text{0.021}}$
& $4.062\pms{\text{0.012}}$ & $3.823\pms{\text{0.114}}$
& $2.144\pms{\text{0.014}}$ & $2.987\pms{\text{0.053}}$
& $1.419\pms{0.032}$ & $7.448\pms{0.121}$ \\

\textbf{HyperIMTS}
& $0.031\pms{0.013}$ & $0.322\pms{0.021}$
& $0.234\pms{0.003}$ & $2.719\pms{0.020}$
& $\best{3.820}\pms{\text{0.018}}$ & $\best{3.721}\pms{\text{0.137}}$
& $\underline{2.055}\pms{\text{0.015}}$ & $\best{2.739}\pms{\text{0.021}}$
& OOM & OOM \\

\midrule

\textbf{CVTA}
& $\best{0.003}\pms{0.002}$ & $\best{0.112}\pms{\text{0.033}}$
& $\best{0.226}\pms{\text{0.009}}$ & $\best{2.543}\pms{\text{0.024}}$
& $\underline{3.957}\pms{\text{0.009}}$ & $\underline{3.778}\pms{\text{0.016}}$
& $\best{2.026}\pms{\text{0.006}}$ & $\underline{2.811}\pms{\text{0.016}}$
& $\best{0.513}\pms{\text{0.018}}$ & $\best{3.941}\pms{\text{0.034}}$ 
\\
\bottomrule
\end{tabular}
\vspace{-1em}
\label{table:nextObsPrediction}
\end{table*}

\section{Results and Discussion}
In this section, we present empirical results of CVTA and multiple baselines on real-world irregular time series prediction benchmarks. Detailed analyses and discussion of the results follow. 

\textbf{Main Results.} \Cref{table:nextObsPrediction} reports and compares the predictive performance of CVTA with state-of-the-art irregular time series models on the next observation prediction task across five diverse benchmarks. 
Despite its recency-based state design, CVTA consistently delivers strong performance, achieving the best prediction error on three out of five datasets and remaining competitive on the other two.
Notably, CVTA yields the largest improvements on the challenging MIMIC-HiRes dataset, a high-dimensional NMAR (Not Missing At Random \cite{Ramoni1997Missingness,Singh1997Missingness}) setting, characterized by fine-grained temporal resolution and exhibiting the highest sparsity among all evaluated datasets. From all the recurrent models, GRUwE is competitive, while differential equation‑based approaches generally underperform relative to other baselines. Among the attention models, the message-passing attention models (HyperIMTS and GraFITi) offer competitive predictive performance.

\textbf{CVTA Ablation Studies.} CVTA architecture relies on the CVTA state and a prediction function implemented using the cross-variable temporal attention mechanism. To understand  better where CVTA's gains come from, we perform three sets of ablations: (i) CVTA state component ablations, (ii) CVTA prediction function ablations, and (iii) other CVTA ablations.

\textit{CVTA State Component Ablations.} We explore the contribution of each CVTA state component by replacing it with an all-zero input before embedding (Appendix~\ref{section:state_component_ablation}), keeping the other state components and the prediction function fixed. Removing the recent observation value causes the largest degradation, followed by elapsed time, and the full CVTA state matches or outperforms every ablated variant.

\textit{CVTA Prediction Function Ablations.} We explore the contribution of the prediction function by feeding the same CVTA state to non-neural (simple persistence and linear extrapolation), FFN, and attention baselines (Appendix~\ref{section:contribution_of_CVTA}). CVTA outperforms all of them on USHCN, Physionet, and MIMIC-III. 
This shows that the CVTA state alone is not sufficient and that cross-variable temporal attention is effective in extracting the prediction signal.

\textit{Other CVTA ablations.} We ablate the mechanisms CVTA uses to address multiple core challenges of irregular multivariate time series prediction: variable heterogeneity, by replacing VCEs with parameter-matched shared embeddings (`w/o VCE'); local temporal dynamics, by removing the local-slope branch (`w/o local slope'); missingness, by removing mask embeddings from the keys (`w/o key mask') or values (`w/o value mask'); and irregular timing, by zeroing all elapsed times (`w/o elapsed time'). We include a detailed description of the ablated CVTA variants and their results in the Appendix~\ref{appendix:cvta_ablation_result}. \Cref{tab:ablation_mse_drop_avg} quantifies the performance impact of each ablation by reporting both absolute error and relative degradation. Overall, every ablation reduces predictive performance, confirming that each ablated design choice contributes to CVTA’s effectiveness. The largest degradation occurs when VCE is replaced by shared embedding, with an average performance drop of 125\%, demonstrating that modeling variable-specific heterogeneity is central to CVTA’s gains. 
Removing value-mask embeddings results in an average degradation of 61\%, showing that explicit missingness information in the value representation is important. In contrast, removing key-mask embeddings causes a smaller but consistent degradation, suggesting that CVTA can partially recover missingness information through the value-mask and elapsed-time embeddings.

\begin{table*}[t]
\centering
\scriptsize
\caption{CVTA architectural ablation study. We report MSE (mean $\pm$ std) and relative MSE performance drop in percentage ($\downarrow$\%), relative to CVTA's. Avg.\ $\downarrow$\% is the average MSE degradation across all the datasets.}
\label{tab:ablation_mse_drop_avg}
\setlength{\tabcolsep}{2pt} 
\begin{tabular}{l|cc|cc|cc|cc|cc|c}
\toprule
\textbf{Model}
& \textbf{USHCN} & \textbf{$\downarrow$\%}
& \textbf{Physionet}  & \textbf{$\downarrow$\%}
& \textbf{MIMIC-III } & \textbf{$\downarrow$\%}
& \textbf{MIMIC-IV}  & \textbf{$\downarrow$\%}
& \textbf{MIMIC-HiRes}  & \textbf{$\downarrow$\%}
& \textbf{Avg.\ $\downarrow$\%} \\
\midrule
\textbf{CVTA}
& $\best{0.003}\pms{0.002}$ & --
& $\best{0.226}\pms{0.009}$ & --
& $\best{3.957}\pms{0.009}$ & --
& $\best{2.026}\pms{0.006}$ & --
& $\best{0.513}\pms{0.018}$ & --
& -- \\
\lightmidrule
w/o VCE
& $0.013\pms{0.003}$ & 333.3
& $0.494\pms{0.008}$ & 118.6
& $5.236\pms{0.008}$ & 32.3
& $2.932\pms{0.006}$ & 44.8
& $1.016\pms{0.036}$ & 98.1
& \best{125.4} \\

w/o local slope
& $0.006\pms{0.004}$ & 100.0
& $0.241\pms{0.009}$ & 6.6
& $4.005\pms{0.009}$ & 1.2
& $2.129\pms{0.009}$ & 5.1
& $0.903\pms{0.018}$ & 76.0
& 37.8 \\

w/o key mask
& $0.004\pms{0.003}$ & 33.3
& $0.227\pms{0.010}$ & 0.4
& $3.960\pms{0.009}$ & 0.1
& $2.080\pms{0.006}$ & 2.7
& $0.809\pms{0.016}$ & 57.7
& 18.8 \\

w/o value mask
& $0.010\pms{0.001}$ & 233.3
& $0.237\pms{0.007}$ & 4.9
& $4.110\pms{0.007}$ & 3.9
& $2.117\pms{0.006}$ & 4.5
& $0.832\pms{0.018}$ & 62.2
& \underline{61.8} \\

w/o elapsed time
& $0.004\pms{0.003}$ & 33.3
& $0.241\pms{0.004}$ & 6.6
& $4.112\pms{0.008}$ & 3.9
& $2.168\pms{0.005}$ & 7.0
& $0.869\pms{0.019}$ & 69.4
& 24.0 \\
\bottomrule
\end{tabular}
\vspace{-2em}
\end{table*}

\textbf{Integration of CVTA with other prediction models.}
As demonstrated in \Cref{table:nextObsPrediction}, CVTA's focus on recent cross-variable temporal patterns provides a strong short-term advantage, allowing it to outperform or remain competitive with long-context models such as GRUwE, GraFITi, and HyperIMTS. Since these models already summarize broader temporal context, we hypothesize that CVTA can complement them by modeling local cross-variable dynamics. To test this, we integrate CVTA as an auxiliary predictor into each base model. For a fair comparison, we use parameter-matched base models with at least as many parameters as their CVTA-augmented counterparts. \Cref{tab:cvta_extension_main_paper} summarizes the performance of CVTA, the parameter-matched base models, and their CVTA-augmented variants across all datasets and prediction horizons. The results show two clear patterns: CVTA is strongest at shorter horizons, consistent with its local CVTA state design, and adding CVTA consistently improves every base model across longer horizons on all real-world datasets. The gains are largest on challenging datasets with long, sparse, and high-dimensional sequences, such as USHCN and MIMIC-HiRes. On Physionet, MIMIC-III, and MIMIC-IV, the improvements are smaller but still consistent across base models and horizons. These results confirm that CVTA provides complementary short-term predictive signal even when the base model is strong. We include additional implementation details in the Appendix~\ref{section:cvta_base_model_integration}.

\begin{table*}[t]
\tiny
\centering
\setlength{\tabcolsep}{1.5pt}
\caption{Multi‑step prediction performance in MSE (we report mean MSE over five random seeds) across all datasets for prediction horizons $\mathrm{H} = \{1,3,5,7\}$ steps. '--' indicates GPU memory error encountered during training.}
\label{tab:cvta_extension_main_paper}
\begin{tabular}{l @{\hspace{1em}} cccc cccc cccc cccc cccc}
\toprule
\textbf{Method} 
& \multicolumn{4}{c}{\textbf{USHCN}}
& \multicolumn{4}{c}{\textbf{Physionet}}
& \multicolumn{4}{c}{\textbf{MIMIC-III}}
& \multicolumn{4}{c}{\textbf{MIMIC-IV}}
& \multicolumn{4}{c}{\textbf{MIMIC-HiRes}} \\
\cmidrule(lr){2-5}\cmidrule(lr){6-9}\cmidrule(lr){10-13}\cmidrule(lr){14-17}\cmidrule(lr){18-21}
&  H=1 & H=3 & H=5 & H=7
  & H=1 & H=3 & H=5 & H=7
  & H=1 & H=3 & H=5 & H=7
  & H=1 & H=3 & H=5 & H=7
  & H=1 & H=3 & H=5 & H=7 \\
\midrule

\textbf{CVTA}  &
\textbf{0.003} & \textbf{0.010} & \textbf{0.013} & \textbf{0.015} &
0.226 & 0.245 & 0.254 & 0.274 &
3.957 & 4.155 & 4.282 & 4.435 &
2.026 & 2.259 & 2.382 & 2.452 &
\textbf{0.513} & \textbf{0.832} & \textbf{0.913} & 1.041\\
\lightmidrule
\textbf{GRUwE (PM)} &
0.011 & 0.033 & 0.118 & 0.184 &
0.234 & 0.250 & 0.262 & 0.279 &
4.010 & 4.113 & 4.216 & 4.322 &
2.110 & 2.287 & 2.376 & 2.461 &
0.809 & 1.224 & 1.326 & 1.329 \\

\textbf{\textsc{GRUwE+CVTA}}  &
0.004 & 0.017 & 0.024 & 0.028 &
\textbf{0.211} & \textbf{0.230} & 0.239 & 0.258 &
3.783 & 3.925 & 4.039 & 4.151 &
2.015 & 2.191 & 2.285 & 2.334 &
0.604 & 0.931 & 0.959 & \textbf{0.989} \\
\textbf{Improv. ($\uparrow$\%)}&
+63.6 & +48.5 & +79.7 & +84.8 & +9.8 & +8.0 & +8.8 & +7.5 &
+5.7 & +4.6 & +4.2 & +4.0 &
+4.5 & +4.2 & +3.8 & +5.2 &
+25.3 & +23.9 & +27.7 & +25.6 \\
\lightmidrule
\textbf{GraFITi (PM)} &
0.072 & 0.081 & 0.099 & 0.113 &
0.234 & 0.243 & 0.247 & 0.256 &
4.059 & 4.083 & 4.216 & 4.264 &
2.146 & 2.277 & 2.370 & 2.371 &
1.409 & 1.445 & 1.462 & 1.703 \\

\textbf{\textsc{GraFITi+CVTA}} &
0.013 & 0.020 & 0.042 & 0.044 &
0.225 & 0.235 & \textbf{0.235} & 0.248 &
3.960 & 3.962 & 4.029 & 4.111 &
2.037 & 2.209 & 2.264 & 2.269 &
0.769 & 1.001 & 1.039 & 1.248 \\

\textbf{Improv. ($\uparrow$\%)} &
+81.9 & +75.3 & +57.6 & +61.1 & +3.8 & +3.3 & +4.9 & +3.1  & +2.4 & +3.0 & +4.4 & +3.6 &
+5.1 & +3.0 & +4.5 & +4.3 &
+45.4 & +30.7 & +28.9 & +26.7 \\
\lightmidrule
\textbf{HyperIMTS (PM)}  &
0.032 & 0.034 & 0.045 & 0.046 &
0.231 & 0.242 & 0.247 & 0.253 &
3.817 & 3.886 & 3.913 & 3.942 &
2.054 & 2.142 & 2.208 & 2.266 &
-- & -- & -- & -- \\

\textbf{\textsc{HyperIMTS+CVTA}}  &
0.010 & 0.011 & \textbf{0.013} & 0.016 &
0.226 & 0.234 & \textbf{0.235} & \textbf{0.247} &
\textbf{3.755} & \textbf{3.759} & \textbf{3.816} & \textbf{3.939} &
\textbf{1.978} & \textbf{2.119} & \textbf{2.179} & \textbf{2.215} &
-- & -- & -- & -- \\
\textbf{Improv. ($\uparrow$\%)} & 
+68.8 & +67.6 & +71.1 & +65.2 & +2.2 & +3.3 & +4.9 & +2.4 &
+1.6 & +3.3 & +2.5 & +0.1 &
+3.7 & +1.1 & +1.3 & +2.3 &
-- & -- & -- & -- \\

\bottomrule
\end{tabular}
\end{table*}
\textbf{Skill Deterioration Analysis.}
The results in \Cref{table:nextObsPrediction} show CVTA is able to support accurate short-term predictions. \textit{But what is the effective horizon of accurate CVTA's predictions?}  To explore this question, we adapt the forecast skill horizon framework \cite{ForecastSkillHorizon} on Physionet data and analyze the skill deterioration horizon of CVTA on individual variables against the strongest baselines (see  Appendix~\ref{section:skill_deterioration_horizon}). 
\cref{fig:main_paper_efficiency_attn_scores}(c) shows that CVTA's skill horizon varies across variables: recent context is sufficient for at least 6 hours for Na and WBC, lasts for an intermediate range for BUN, pH, and Temp, and deteriorates earlier for Creatinine, GCS, and HCT, suggesting that longer temporal context may improve predictions for these variables.

\textbf{Computational Efficiency.} 
CVTA's state is designed for efficient computation in both batched settings, where the full sequence is available, and online settings, where observations arrive sequentially over time. 

\textit{Batched Training.} During batched training, unlike recurrent models \cite{CRUPaper,GRUDPaper,GRUwE}, all the CVTA state components can be computed in parallel using vectorized operations (see Appendix~\ref{appendix:cvta_markov_state_computation} for implementation details) over the full observation sequence, yielding training times competitive with existing attention-based models. \Cref{fig:main_paper_efficiency_attn_scores}(a) illustrates the trade-off between training speed, memory use, and predictive accuracy on the Physionet benchmark. The plot reveals three distinct efficiency profiles: (i) differential‑equation models use minimal memory but train slowly; (ii) attention models train quickly at a high memory cost; and (iii) recurrent models occupy a middle ground with moderate memory and speed. CVTA resides in the high‑performance region of this Pareto frontier, matching the training efficiency of attention‑based approaches while delivering superior predictive performance. 

\textit{Online Inference.} During online inference with streaming observations, the CVTA state design supports lightweight and recursive state updates using a fixed set of equations (i.e., \cref{eq:update_x,eq:update_mask,eq:update_delta,eq:slope_mask,eq:update_slope,eq:slope_mask_update}), that are triggered as each new observation arrives. Since each variable requires
constant work, updating the CVTA state costs $\mathcal{O}(F)$ computation and memory. Predicting all $F$ variables costs $\mathcal{O}(F^2D + FD^2)$ computation, from the cross-variable attention branches and the FFN layers respectively, with $\mathcal{O}(F^2 + FD)$ working memory for the attention scores and variable embeddings. Importantly, neither cost depends on the length of the observation history, so per-step cost stays constant as the stream grows.
We include the results of our online inference computation experiment in the Appendix~\ref{appendix:cvta_markov_state_computation} and visualize the results in \Cref{fig:main_paper_efficiency_attn_scores}(b). Together, these properties make CVTA well-suited for both batch training on historical data and low-latency, real-time deployment (e.g., ICU monitoring).

\begin{figure}
    \centering
    \includegraphics[width=0.99\linewidth]{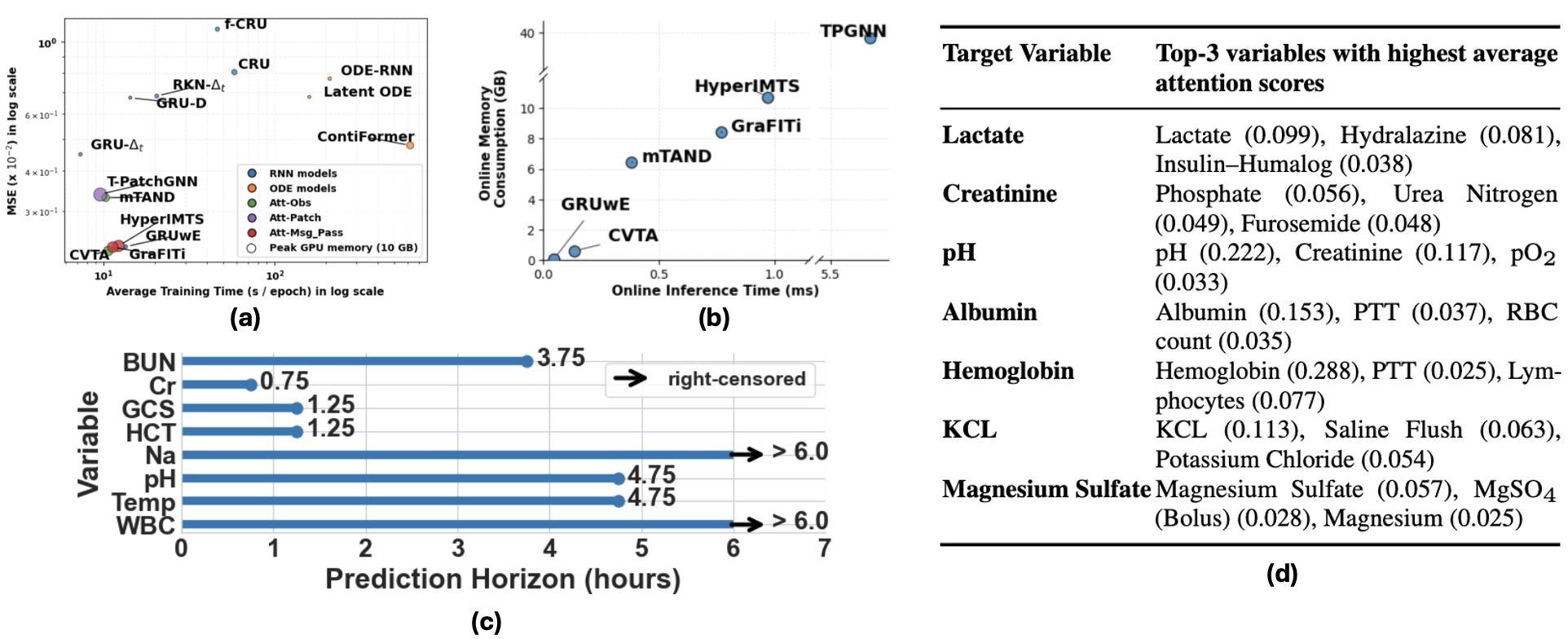}
    \caption{\textbf{(a)} Trade-off between training efficiency, predictive performance, and peak memory usage during batched training. \textbf{(b)} Efficiency profile of online inference across top-performing models. \textbf{(c)} CVTA's skill deterioration horizon by variable ($\Delta T^*_{\min}$). \textbf{(d)} Top-attended variables and average attention scores.}
    \label{fig:main_paper_efficiency_attn_scores}
\end{figure}

\textbf{Interpretability.} We evaluate CVTA's interpretability using the MIMIC-III dataset \cite{GRUODEBayes,MIMICIIIDataset}, which contains laboratory results and medication records. As CVTA generates predictions using interpretable attention scores, we analyze which variables receive the highest attention when forecasting specific clinical measurements and interventions. We summarize the results of this analysis in the Figure~\ref{fig:main_paper_efficiency_attn_scores}(d). CVTA's attention profiles are clinically meaningful: predictions often show strong self-attention, as in pH and Hemoglobin, reflecting local temporal continuity; capture known physiological links, such as Phosphate and Urea Nitrogen with Creatinine for renal monitoring; and connect semantically equivalent variables, such as KCL and Potassium Chloride. Most high-attention links align with clinical intuition, while less-expected associations, such as PTT and Albumin, may reflect limited data or unobserved confounding rather than causal physiology. The learned variable embeddings for MIMIC-III (\cref{fig:mimic_iii_emb}) and MIMIC-IV (\cref{fig:mimic_iv_emb}) also cluster clinically similar variables together.

\textbf{Interpretability Case Studies.} We further examine two instance-level clinical scenarios involving high lactate and high creatinine predictions to assess the clinical plausibility of CVTA's attention patterns. In both cases, the attended variables align with plausible physiological mechanisms, including acid--base regulation, oxygenation, perfusion, renal filtration, and systemic illness burden. Detailed case-study descriptions, attention scores, and clinical interpretations are provided in Appendix~\ref{section:clinical_case_studies}.

\textbf{Attention Weight Faithfulness.}
Since we use CVTA's attention weights for interpretability, we test if these weights faithfully reflect the variables that influence the model’s predictions. Following prior work \cite{AttNotExp,SerranoSmith}, we perform attention-weight erasure by masking the top-3 or bottom-3 attended variables, renormalizing the remaining weights, and measuring the paired change in MSE. On Physionet, masking the top-3 variables increases MSE by $0.0417 \pm 0.0003$, while masking the bottom-3 variables produces a near-zero change of $-1.12\times 10^{-5} \pm 0.07\times 10^{-5}$, suggesting that CVTA's attention weights are indeed faithful to prediction. Details are provided in the Appendix~\ref{section:attn_weights_faithfulness}.

\begin{wrapfigure}{r}{0.45\linewidth}
    \vspace{-1em}
    \centering
    \includegraphics[width=\linewidth]{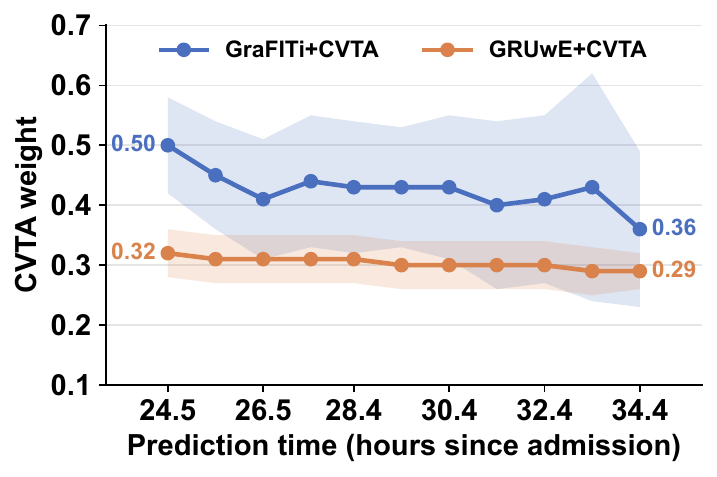}
    \caption{Meta-decision weight for CVTA across horizons when integrated with GRUwE and GraFITi on Physionet.}
    \label{fig:meta_decision_weights}
\end{wrapfigure}
\textbf{Meta-Decision Model Behavior.} To characterize the meta-decision model behavior, we train GRUwE+CVTA and GraFITi+CVTA on Physionet with a longer prediction horizon of $H=7$ steps. We then bin predictions by forecast horizon and compute the mean weight assigned to CVTA within each bin (\cref{fig:meta_decision_weights}; details in Appendix~\ref{appendix:meta_decision}). CVTA receives substantial weight at every horizon (0.30--0.50), and this weight decreases as the horizon grows, by 28\% for GraFITi+CVTA and 9.4\% for GRUwE+CVTA, indicating that the meta-decision model assigns more weight to CVTA for near-term predictions and shifts toward the base model at longer horizons. This is consistent with the skill deterioration analysis (Appendix~\ref{section:skill_deterioration_horizon}), where CVTA is most competitive at shorter prediction horizons. The stronger decline for GraFITi+CVTA further suggests that this adaptation depends on the complementary strengths of each base model.

\section{Conclusion}
In this work, we introduced Cross-Variable Temporal Attention (CVTA), a short-term forecasting framework for multivariate irregular time series. CVTA uses a compact, interpretable state, defined by most recent observations and their trends. The prediction function is defined by variable-conditioned embeddings and cross-variable temporal attention mechanism. Experiments across diverse benchmarks show that CVTA achieves strong predictive performance, with variable-conditioned embeddings playing a central role in its gains. We further show that CVTA can serve as a complementary recency-based prediction model for state-of-the-art models through a dynamic meta-decision mechanism, consistently improving their performance.

\section*{Acknowledgements and Disclosures}
The research work in this paper was supported in part by NIH grant R01 EB032752, and by the University of Pittsburgh Center for Research Computing and Data resource RRID:SCR 022735 using HTC cluster, which is supported by NIH award S10OD028483. The content of this paper is solely the responsibility of the authors and does not necessarily represent the official views of the NIH.

\bibliography{paper}

@inproceedings{AttNotExp,
  title={Attention is not explanation},
  author={Jain, Sarthak and Wallace, Byron C},
  booktitle={Proceedings of the 2019 Conference of the North American Chapter of the Association for Computational Linguistics: Human Language Technologies, Volume 1 (Long and Short Papers)},
  pages={3543--3556},
  year={2019}
}

@inproceedings{SerranoSmith,
  title={Is attention interpretable?},
  author={Serrano, Sofia and Smith, Noah A},
  booktitle={Proceedings of the 57th annual meeting of the association for computational linguistics},
  pages={2931--2951},
  year={2019}
}

@article{ForecastSkillHorizon,
  title={The forecast skill horizon},
  author={Buizza, Roberto and Leutbecher, Martin},
  journal={Quarterly Journal of the Royal Meteorological Society},
  volume={141},
  number={693},
  pages={3366--3382},
  year={2015},
  publisher={Wiley Online Library}
}

@InProceedings{CRUPaper,
  title = 	 {Modeling Irregular Time Series with Continuous Recurrent Units},
  author =       {Schirmer, Mona and Eltayeb, Mazin and Lessmann, Stefan and Rudolph, Maja},
  booktitle = 	 {Proceedings of the 39th International Conference on Machine Learning},
  pages = 	 {19388--19405},
  year = 	 {2022},
  editor = 	 {Chaudhuri, Kamalika and Jegelka, Stefanie and Song, Le and Szepesvari, Csaba and Niu, Gang and Sabato, Sivan},
  volume = 	 {162},
  series = 	 {Proceedings of Machine Learning Research},
  month = 	 {17--23 Jul},
  publisher =    {PMLR},
  url = 	 {https://proceedings.mlr.press/v162/schirmer22a.html}
}

@misc{GRUPaper,
  doi = {10.48550/ARXIV.1412.3555},
  url = {https://arxiv.org/abs/1412.3555},
  author = {Chung,  Junyoung and Gulcehre,  Caglar and Cho,  KyungHyun and Bengio,  Yoshua},
  title = {Empirical Evaluation of Gated Recurrent Neural Networks on Sequence Modeling},
  publisher = {arXiv},
  year = {2014},
  copyright = {arXiv.org perpetual,  non-exclusive license}
}

@article{GRUDPaper,
  title = {Recurrent Neural Networks for Multivariate Time Series with Missing Values},
  volume = {8},
  ISSN = {2045-2322},
  url = {http://dx.doi.org/10.1038/s41598-018-24271-9},
  DOI = {10.1038/s41598-018-24271-9},
  number = {1},
  journal = {Scientific Reports},
  publisher = {Springer Science and Business Media LLC},
  author = {Che,  Zhengping and Purushotham,  Sanjay and Cho,  Kyunghyun and Sontag,  David and Liu,  Yan},
  year = {2018},
  month = apr 
}

@inproceedings{LatentODEPaper,
 author = {Rubanova, Yulia and Chen, Ricky T. Q. and Duvenaud, David K},
 booktitle = {Advances in Neural Information Processing Systems},
 editor = {H. Wallach and H. Larochelle and A. Beygelzimer and F. d\textquotesingle Alch\'{e}-Buc and E. Fox and R. Garnett},
 pages = {},
 publisher = {Curran Associates, Inc.},
 title = {Latent Ordinary Differential Equations for Irregularly-Sampled Time Series},
 url = {https://proceedings.neurips.cc/paper_files/paper/2019/file/42a6845a557bef704ad8ac9cb4461d43-Paper.pdf},
 volume = {32},
 year = {2019}
}

@inproceedings{ContiFormer,
  title={ContiFormer: Continuous-Time Transformer for Irregular Time Series Modeling},
  author={Chen, Yuqi and Ren, Kan and Wang, Yansen and Fang, Yuchen and Sun, Weiwei and Li, Dongsheng},
  booktitle={Thirty-seventh Conference on Neural Information Processing Systems},
  year={2023}
}

@inproceedings{mTANDPaper,
title={Multi-Time Attention Networks for Irregularly Sampled Time Series},
author={Satya Narayan Shukla and Benjamin Marlin},
booktitle={International Conference on Learning Representations},
year={2021},
url={https://openreview.net/forum?id=4c0J6lwQ4_}
}

@inproceedings{GRUODEBayes,
 author = {De Brouwer, Edward and Simm, Jaak and Arany, Adam and Moreau, Yves},
 booktitle = {Advances in Neural Information Processing Systems},
 editor = {H. Wallach and H. Larochelle and A. Beygelzimer and F. d\textquotesingle Alch\'{e}-Buc and E. Fox and R. Garnett},
 pages = {},
 publisher = {Curran Associates, Inc.},
 title = {GRU-ODE-Bayes: Continuous Modeling of Sporadically-Observed Time Series},
 url = {https://proceedings.neurips.cc/paper_files/paper/2019/file/455cb2657aaa59e32fad80cb0b65b9dc-Paper.pdf},
 volume = {32},
 year = {2019}
}

@inproceedings{PhysionetDataset,
  title={Predicting in-hospital mortality of icu patients: The physionet/computing in cardiology challenge 2012},
  author={Silva, Ikaro and Moody, George and Scott, Daniel J and Celi, Leo A and Mark, Roger G},
  booktitle={2012 Computing in Cardiology},
  pages={245--248},
  year={2012},
  organization={IEEE}
}

@misc{MIMICIIIDataset,
  doi = {10.13026/CD7Z-WG25},
  url = {https://physionet.org/content/mimiciii/},
  author = {Johnson,  Alistair and Pollard,  Tom and Mark,  Roger},
  title = {MIMIC-III Clinical Database},
  publisher = {PhysioNet},
  year = {2023}
}

@techreport{USHCNDataset,
  title={Long-term daily and monthly climate records from stations across the contiguous United States (US Historical Climatology Network)},
  author={Menne, MJ and Williams Jr, CN and Vose, RS},
  year={2016},
  institution={Environmental System Science Data Infrastructure for a Virtual Ecosystem~…}
}

@inproceedings{AttentionPaper,
 author = {Vaswani, Ashish and Shazeer, Noam and Parmar, Niki and Uszkoreit, Jakob and Jones, Llion and Gomez, Aidan N and Kaiser, \L ukasz and Polosukhin, Illia},
 booktitle = {Advances in Neural Information Processing Systems},
 editor = {I. Guyon and U. Von Luxburg and S. Bengio and H. Wallach and R. Fergus and S. Vishwanathan and R. Garnett},
 pages = {},
 publisher = {Curran Associates, Inc.},
 title = {Attention is All you Need},
 url = {https://proceedings.neurips.cc/paper_files/paper/2017/file/3f5ee243547dee91fbd053c1c4a845aa-Paper.pdf},
 volume = {30},
 year = {2017}
}

@article{NeuralODE,
  title={Neural ordinary differential equations},
  author={Chen, Ricky TQ and Rubanova, Yulia and Bettencourt, Jesse and Duvenaud, David K},
  journal={Advances in neural information processing systems},
  volume={31},
  year={2018}
}

@article{NeuralCDEPaper,
  title={Neural controlled differential equations for irregular time series},
  author={Kidger, Patrick and Morrill, James and Foster, James and Lyons, Terry},
  journal={Advances in Neural Information Processing Systems},
  volume={33},
  pages={6696--6707},
  year={2020}
}

@inproceedings{Raindrop,
title = {Graph-Guided Network For Irregularly Sampled Multivariate Time Series},
author = {Zhang, Xiang and Zeman, Marko and Tsiligkaridis, Theodoros and Zitnik, Marinka},
booktitle = {International Conference on Learning Representations, ICLR},
year      = {2022}
}

@article{NeuralHawkes,
  title={The neural hawkes process: A neurally self-modulating multivariate point process},
  author={Mei, Hongyuan and Eisner, Jason M},
  journal={Advances in neural information processing systems},
  volume={30},
  year={2017}
}

@inproceedings{TransformerHawkes,
  title={Transformer hawkes process},
  author={Zuo, Simiao and Jiang, Haoming and Li, Zichong and Zhao, Tuo and Zha, Hongyuan},
  booktitle={International conference on machine learning},
  pages={11692--11702},
  year={2020},
  organization={PMLR}
}

@InProceedings{RKNPaper,
  title = 	 {Recurrent Kalman Networks: Factorized Inference in High-Dimensional Deep Feature Spaces},
  author =       {Becker, Philipp and Pandya, Harit and Gebhardt, Gregor and Zhao, Cheng and Taylor, C. James and Neumann, Gerhard},
  booktitle = 	 {Proceedings of the 36th International Conference on Machine Learning},
  pages = 	 {544--552},
  year = 	 {2019},
  editor = 	 {Chaudhuri, Kamalika and Salakhutdinov, Ruslan},
  volume = 	 {97},
  series = 	 {Proceedings of Machine Learning Research},
  month = 	 {09--15 Jun},
  publisher =    {PMLR},
  url = 	 {https://proceedings.mlr.press/v97/becker19a.html}
}

@article{Singh1997Missingness,
  title={Learning Bayesian networks from incomplete data},
  author={Singh, Moninder},
  journal={AAAI/IAAI},
  volume={1001},
  pages={534--539},
  year={1997},
  publisher={Citeseer}
}

@inproceedings{Ramoni1997Missingness,
  title={Robust parameter learning in Bayesian networks with missing data},
  author={Ramoni, Marco and Sebastiani, Paola},
  booktitle={Sixth International Workshop on Artificial Intelligence and Statistics},
  pages={399--406},
  year={1997},
  organization={PMLR}
}

@inproceedings{TPGNN,
  title={Irregular Multivariate Time Series Forecasting: A Transformable Patching Graph Neural Networks Approach},
  author={Zhang, Weijia and Yin, Chenlong and Liu, Hao and Zhou, Xiaofang and Xiong, Hui},
  booktitle={Forty-first International Conference on Machine Learning},
  year={2024}
}

@inproceedings{grafitiPaper,
  title={GraFITi: Graphs for Forecasting Irregularly Sampled Time Series},
  author={Yalavarthi, Vijaya Krishna and Madhusudhanan, Kiran and Scholz, Randolf and Ahmed, Nourhan and Burchert, Johannes and Jawed, Shayan and Born, Stefan and Schmidt-Thieme, Lars},
  booktitle={Proceedings of the AAAI Conference on Artificial Intelligence},
  volume={38},
  number={15},
  pages={16255--16263},
  year={2024}
}

@inproceedings{SAHP,
  title={Self-attentive Hawkes process},
  author={Zhang, Qiang and Lipani, Aldo and Kirnap, Omer and Yilmaz, Emine},
  booktitle={International conference on machine learning},
  pages={11183--11193},
  year={2020},
  organization={PMLR}
}

@inproceedings{RMTPP,
  title={Recurrent marked temporal point processes: Embedding event history to vector},
  author={Du, Nan and Dai, Hanjun and Trivedi, Rakshit and Upadhyay, Utkarsh and Gomez-Rodriguez, Manuel and Song, Le},
  booktitle={Proceedings of the 22nd ACM SIGKDD international conference on knowledge discovery and data mining},
  pages={1555--1564},
  year={2016}
}

@inproceedings{HyperIMTS,
  title={HyperIMTS: Hypergraph Neural Network for Irregular Multivariate Time Series Forecasting},
  author={Li, Boyuan and Luo, Yicheng and Liu, Zhen and Zheng, Junhao and Lv, Jianming and Ma, Qianli},
  booktitle={Forty-second International Conference on Machine Learning},
  year={2025}
}

@inproceedings{hiPatch,
  title={Hi-Patch: Hierarchical Patch GNN for Irregular Multivariate Time Series},
  author={Luo, Yicheng and Zhang, Bowen and Liu, Zhen and Ma, Qianli},
  booktitle={Forty-second International Conference on Machine Learning},
  year={2025}
}

@misc{
GRUwE,
title={Still Competitive: Revisiting Recurrent Models for Irregular Time Series Prediction},
author={Ankitkumar Joshi and Milos Hauskrecht},
journal={Transactions on Machine Learning Research},
issn={2835-8856},
year={2026},
note={Transactions on Machine Learning Research, 2026},
}

@article{MIMIC_IV,
  title={MIMIC-IV, a freely accessible electronic health record dataset},
  author={Johnson, Alistair EW and Bulgarelli, Lucas and Shen, Lu and Gayles, Alvin and Shammout, Ayad and Horng, Steven and Pollard, Tom J and Hao, Sicheng and Moody, Benjamin and Gow, Brian and others},
  journal={Scientific data},
  volume={10},
  number={1},
  pages={1},
  year={2023},
  publisher={Nature Publishing Group UK London}
}

@inproceedings{JeongMinRecencyPaper,
  title={Recent context-aware LSTM for clinical event time-series prediction},
  author={Lee, Jeong Min and Hauskrecht, Milos},
  booktitle={Conference on Artificial Intelligence in Medicine in Europe},
  pages={13--23},
  year={2019},
  organization={Springer}
}
\bibliographystyle{achemso}
\newpage
\appendix
\addcontentsline{toc}{section}{Appendix} 


\part{Appendix}
\parttoc

\section{Notation}
\label{appendix:notation}
\begin{table}[H]
\centering
\small
\caption{Summary of the main notation used in the paper. $\bullet \in \{r, s\}$ indexes the recent observation and local slope attention branches.}
\label{tab:notation}
\begin{tabular}{lll}
\toprule
\textbf{Symbol} & \textbf{Shape} & \textbf{Meaning} \\
\midrule
$F$ & scalar & Number of variables \\
$D$ & scalar & Embedding dimension \\
$\Delta T$ & $\mathbb{R}_{>0}$ & Prediction horizon \\
\lightmidrule
$\mathcal{S}_{t}$ & $\mathbb{R}^{3F} \times \{0,1\}^{2F}$ & CVTA state at time $t$ \\
$\mathbf{x}_t, \bm{\delta}_t, \dot{\mathbf{x}}_t$ & $\mathbb{R}^F$ & Recent values, elapsed times, local slopes \\
$\mathbf{m}_t, \mathbf{g}_t$ & $\{0,1\}^F$ & Observation and slope masks \\
$\mathbf{o}_t, \mathbf{i}_t$ & $\mathbb{R}^F, \{0,1\}^F$ & New observation and observation indicator \\
$\delta'_{t,f}$ & $\mathbb{R}_{\geq 0}$ & Horizon-adjusted elapsed time \\
\lightmidrule
$\phi^{(f)}_{\bullet}$ & $\mathbb{R} \to \mathbb{R}^D$ & Variable-conditioned embedding operator \\
$\mathbf{e}^{(f)}_{\bullet}$ & $\mathbb{R}^D$ & Embedded state component of variable $f$ \\
$\mathbf{Q}_\bullet, \mathbf{K}_\bullet, \mathbf{V}_\bullet$ & $\mathbb{R}^{F\times D}$ & Queries, keys, and values \\
$\mathbf{H}_\bullet, \tilde{\mathbf{H}}_\bullet$ & $\mathbb{R}^{F\times D}$ & Attention and residual block outputs \\
\lightmidrule
$\hat{\mathbf{y}}_\bullet$ & $\mathbb{R}^F$ & Branch prediction \\
$\hat{\mathbf{y}}_{t+\Delta T}$ & $\mathbb{R}^F$ & Final CVTA prediction \\
\bottomrule
\end{tabular}
\end{table}

\section{Limitation \& Future Work}
As an attention‑based model, CVTA inherits certain architectural constraints. Most notably, its cross‑variable attention scales quadratically in the number of variables, which may become a bottleneck in extremely high‑dimensional settings. While CVTA’s compact state avoids sequence‑length scaling, variable‑count scaling remains. A promising direction to mitigate this limitation is to incorporate domain‑informed attention constraints. By masking out implausible or irrelevant variable dependencies (e.g., based on causal graphs or clinical knowledge), the effective attention complexity can be reduced while simultaneously improving generalization and interpretability. We view the integration of such structured prior knowledge as an exciting avenue for future work.

\section{Broader Impact}
This work develops forecasting methods for irregular multivariate time series, with potential positive impact in domains such as healthcare, climate monitoring, and large-scale sensing, where accurate short-term predictions can support earlier detection and more informed decision-making. However, incorrect predictions in high-stakes settings could lead to inappropriate interventions. Deployment would require careful validation, calibration, monitoring for distribution shift, and domain-expert oversight, especially in clinical applications.

\section{Variable-Conditioned Embeddings}
\subsection{Motivation}
\label{appendix:motivation_variable_cond_emb}
Real-world multivariate time series consist of variables derived from processes with fundamentally different underlying dynamics. Consequently, the observed variables exhibit substantial heterogeneity across several key dimensions: \begin{itemize}[noitemsep,topsep=1pt,partopsep=1pt] 
    \item \textit{Statistical properties}: Variables may be drawn from fundamentally different underlying distributions, exhibiting different ranges, variances, and higher-order moments. 
    \item \textit{Semantic interpretation of values}: The same numerical value may convey different clinical or operational meanings across variables. For example, in a clinical setting, a value considered normal for one vital sign may indicate a critical value for a laboratory recording. 
    \item \textit{Variable-specific noise}: Since multi-variate streams are collected from different sensors, measurement precision, sensor quality and intrinsic noise levels differ across recorded variable streams. 
    \item \textit{Temporal sampling}: Variables are often recorded at different frequencies, with some streams densely sampled and others arriving sporadically or at near-fixed intervals. 
    \item \textit{Predictive influence}: The temporal influence of a variable on future recordings can vary widely. For example, a laboratory result may have a longer predictive relevance, while a vital sign may have a shorter, more limited predictive influence over time. 
\end{itemize}

\vspace{0.5cm}

\subsection{Limitations of Shared Embeddings}
\begin{proposition}[\textbf{Shared embeddings collapse variable identity}]
\label{prop:shared_emb_collapse}
In the recent-observation attention branch of CVTA with shared embeddings, any two source variables with identical scalar CVTA state components receive identical key-value representations.
Consequently, every target query assigns them identical attention logits.
Thus, shared embeddings cannot distinguish variables with identical state values but different semantics or predictive influences.
\end{proposition}

\begin{proof}
Under shared embeddings, the same embedding functions are applied to every variable. Therefore, identical scalar state components imply identical embeddings:
\[
e_x^{(a)} = e_x^{(b)}, 
\qquad
e_m^{(a)} = e_m^{(b)},
\qquad
e_{\delta_r}^{(a)} = e_{\delta_r}^{(b)}.
\]
The recent-observation branch defines keys and values as
\[
k_r^{(f)} = e_{\delta_r}^{(f)} + e_m^{(f)},
\qquad
v_r^{(f)} = e_x^{(f)} + e_m^{(f)}.
\]
Substituting the equal embeddings for variables $a$ and $b$ gives
\[
k_r^{(a)} = k_r^{(b)},
\qquad
v_r^{(a)} = v_r^{(b)}.
\]
Hence, for any target query $q_r^{(f)}$,
\[
(q_r^{(f)})^\top k_r^{(a)}
=
(q_r^{(f)})^\top k_r^{(b)}.
\]
Therefore, under shared embeddings, the attention mechanism cannot distinguish variables $a$ and $b$ through their key-value representations whenever their scalar state components are identical.
\end{proof}

\textbf{Limitations of Shared Embeddings.} In CVTA, each variable is summarized by scalar CVTA state components such as recent value, elapsed time, mask, and slope. If embeddings are shared across variables, then two variables with the same scalar state receive identical key-value representations. That means the attention mechanism cannot distinguish them, even if they represent completely different processes. This is a serious limitation for multivariate time series, where variables can have different semantics, scales, noise levels, sampling frequencies, and predictive roles.

\textbf{Expressivity of VCEs.}
VCEs extend shared embeddings by allowing each variable to learn its own embedding map. Shared embeddings correspond to the special case where these maps are tied across variables, whereas VCEs can assign different representations to the same scalar value depending on the variable identity.

\begin{figure*}[t]
    \centering

    \begin{subfigure}[t]{0.32\textwidth}
        \centering
        \includegraphics[width=\linewidth]{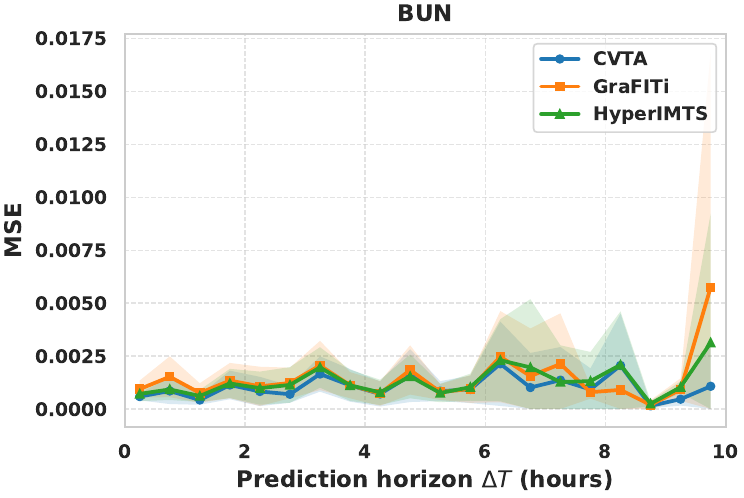}
        \caption{BUN}
        \label{fig:markov_boundary_bun}
    \end{subfigure}
    \hfill
    \begin{subfigure}[t]{0.32\textwidth}
        \centering
        \includegraphics[width=\linewidth]{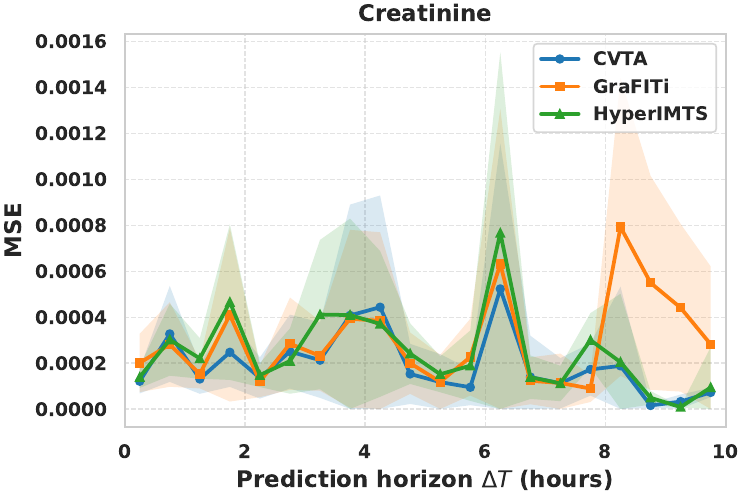}
        \caption{Creatinine}
        \label{fig:markov_boundary_creatinine}
    \end{subfigure}
    \hfill
    \begin{subfigure}[t]{0.32\textwidth}
        \centering
        \includegraphics[width=\linewidth]{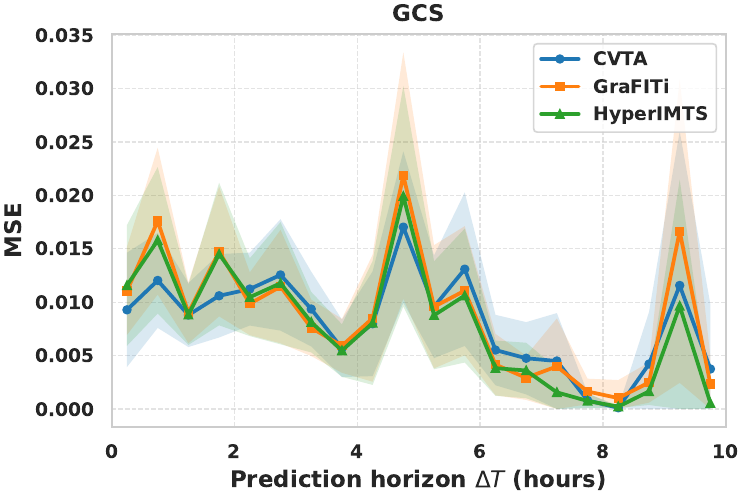}
        \caption{GCS}
        \label{fig:markov_boundary_gcs}
    \end{subfigure}

    \vspace{0.5em}

    \begin{subfigure}[t]{0.32\textwidth}
        \centering
        \includegraphics[width=\linewidth]{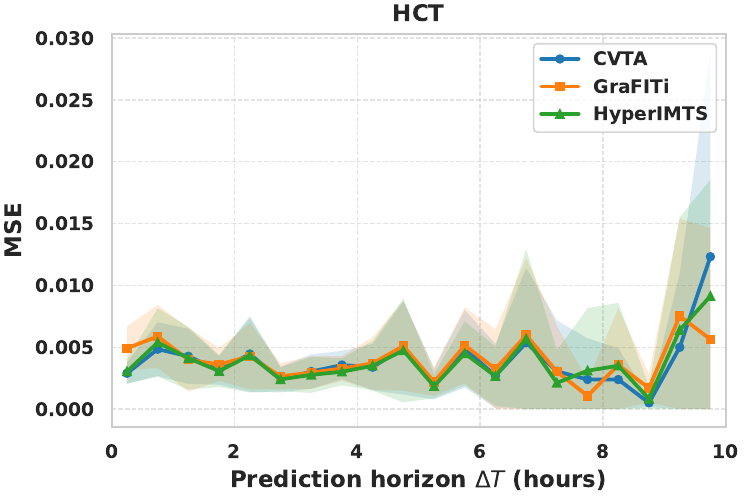}
        \caption{HCT}
        \label{fig:markov_boundary_hct}
    \end{subfigure}
    \hfill
    \begin{subfigure}[t]{0.32\textwidth}
        \centering
        \includegraphics[width=\linewidth]{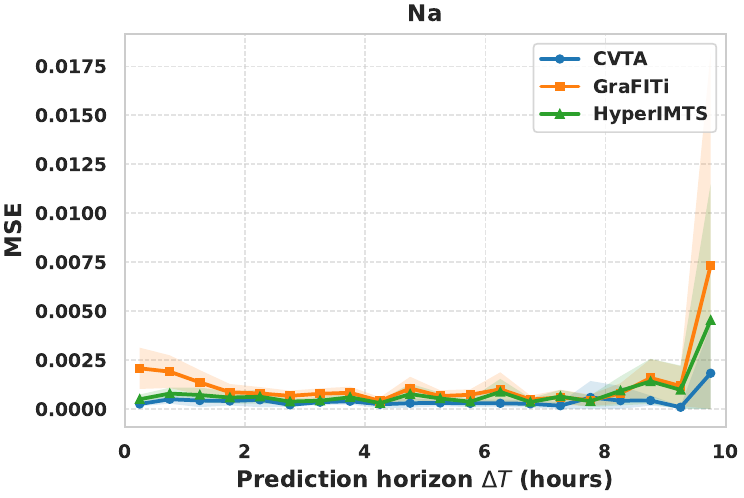}
        \caption{Na}
        \label{fig:markov_boundary_na}
    \end{subfigure}
    \hfill
    \begin{subfigure}[t]{0.32\textwidth}
        \centering
        \includegraphics[width=\linewidth]{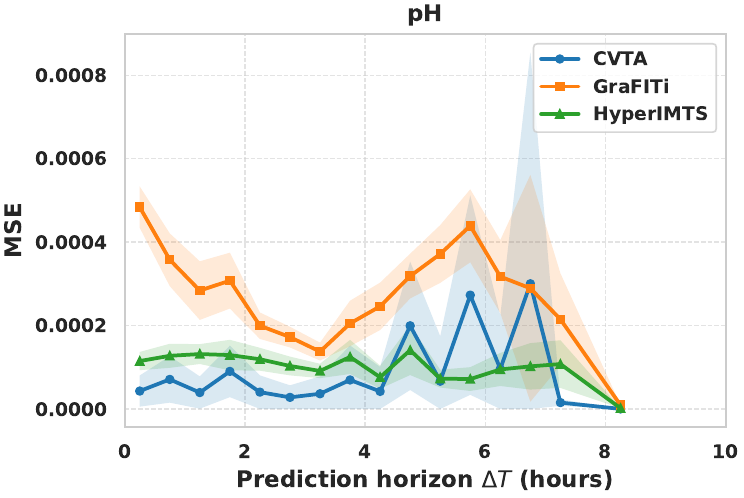}
        \caption{pH}
        \label{fig:markov_boundary_ph}
    \end{subfigure}

    \vspace{0.5em}

    \makebox[\textwidth][c]{%
        \begin{subfigure}[t]{0.32\textwidth}
            \centering
            \includegraphics[width=\linewidth]{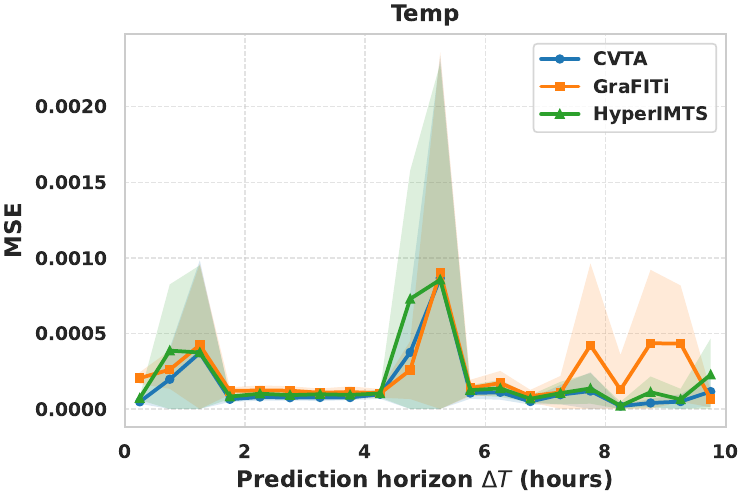}
            \caption{Temp}
            \label{fig:markov_boundary_temp}
        \end{subfigure}
        \hspace{0.03\textwidth}
        \begin{subfigure}[t]{0.32\textwidth}
            \centering
            \includegraphics[width=\linewidth]{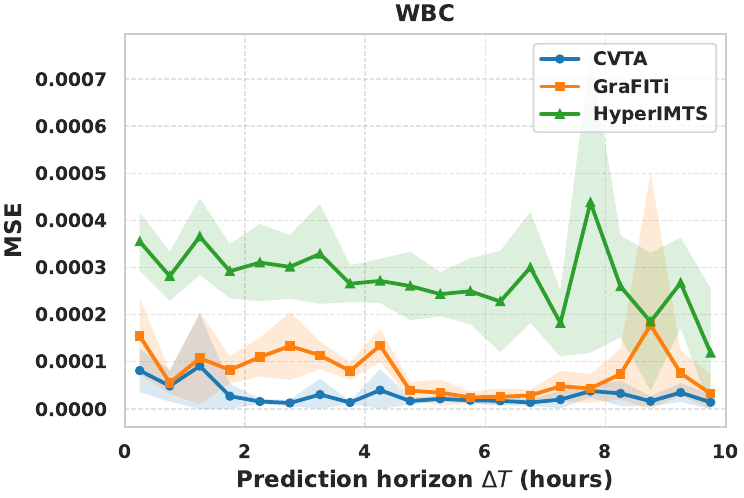}
            \caption{WBC}
            \label{fig:markov_boundary_wbc}
        \end{subfigure}
    }

    \caption{MSE profiles for CVTA, GraFITi and HyperIMTS for representative variables on the Physionet benchmark.}
    \label{fig:markov_boundary_selected_channels}
\end{figure*}

\section{Skill Deterioration Analysis}
\label{section:skill_deterioration_horizon}
CVTA is a short-term prediction model that represents recent observations using a state consisting of recent values, elapsed times, and local slopes. This design enables efficient and interpretable prediction, but it also raises a key question: \textit{for how long does the CVTA state remain sufficiently informative?} To characterize this boundary, we adapt the forecast skill horizon framework \cite{ForecastSkillHorizon} and measure CVTA's forecast skill relative to competitive full-history reference models. For a prediction horizon $\Delta T$ and a trained reference model $\mathrm{ref}$, we define CVTA's forecasting skill as:

\begin{equation}
    \mathrm{Skill}(\Delta T;\mathrm{ref})
    =
    1 -
    \frac{
        \mathrm{MSE}_{\mathrm{CVTA}}(\Delta T)
    }{
        \mathrm{MSE}_{\mathrm{ref}}(\Delta T)
    } .
    \label{eq:forecast_skill}
\end{equation}
Using the forecast skill curve, we define the skill deterioration horizon, $\Delta T^*_{\mathrm{ref}}$, as the earliest prediction horizon at which CVTA no longer outperforms the reference model:
\begin{equation}
    \Delta T^\star_{\mathrm{ref}}
    =
    \inf\Bigl\{
        \Delta T \;:\;
        \mathrm{Skill}(\Delta T; \mathrm{ref}) \le 0
    \Bigr\}.
\end{equation}

On the Physionet holdout set, we evaluate long-term forecasting using models trained with prediction horizon $H=7$ steps. We aggregate predictions from CVTA, GRUwE, GraFITi, and HyperIMTS by variable type and prediction horizon. In particular, for each variable type, we bin valid targets into 30-minute prediction-horizon bins and compute the MSE within each bin. We report the skill deterioration horizon, defined as the earliest bin midpoint at which CVTA no longer has positive skill relative to a reference model. If CVTA maintains positive skill over all valid bins, we report the result as right-censored at the last valid horizon. Due to higher sparsity in some variables and to avoid unstable estimates, bins with fewer than 20 valid targets are excluded from the skill deterioration analysis. 
Table~\ref{tab:markov_boundary_selected_channels} reports the resulting skill deterioration horizons with respect to GRUwE, GraFITi, and HyperIMTS for representative variables. Figure~\ref{fig:markov_boundary_selected_channels} shows the corresponding MSE profiles where shaded regions denote 95\% confidence intervals, computed within each bin as 1.96 times the standard error of the squared errors.

\begin{table}[H]
\centering
\small
\caption{Skill deterioration horizons of CVTA relative to GraFITi, GRUwE, and HyperIMTS models for representative variables on the Physionet benchmark. $\Delta T^\star_{\min}$ denotes the earliest deterioration horizon across the three reference models and provides a conservative estimate of CVTA's competitive prediction horizon.}
\label{tab:markov_boundary_selected_channels}
\begin{tabular}{lcccc}
\toprule
\textbf{Variable} &
$\Delta T^\star_{\mathrm{GraFITi}}$ &
$\Delta T^\star_{\mathrm{GRUwE}}$ &
$\Delta T^\star_{\mathrm{HyperIMTS}}$ &
$\Delta T^\star_{\min}$ \\
&
(hours) &
(hours) &
(hours) &
(hours) \\
\midrule
\textbf{BUN}        & $5.25$ & $5.75$ & $3.75$ & $3.75$ \\
\textbf{Creatinine} & $0.75$ & $>6.0$ & $0.75$ & $0.75$ \\
\textbf{GCS}        & $2.25$ & $1.25$ & $1.25$ & $1.25$ \\
\textbf{HCT}        & $1.25$ & $1.25$ & $1.25$ & $1.25$ \\
\textbf{Na}         & $>6.0$ & $>6.0$ & $>6.0$ & $>6.0$ \\
\textbf{pH}         & $>5.0$ & $>5.0$ & $4.75$ & $4.75$ \\
\textbf{Temp}       & $4.75$ & $4.75$ & $5.25$ & $4.75$ \\
\textbf{WBC}        & $>6.0$ & $>6.0$ & $>6.0$ & $>6.0$ \\
\bottomrule
\end{tabular}
\end{table}
The conservative horizon $\Delta T^\star_{\min}$ shows that CVTA's skill horizon varies across variables, highlighting variable-level heterogeneity. For variables such as Na and WBC, CVTA maintains positive skill against all full-history baselines over all valid bins, indicating that recent values, elapsed times, and local slopes remain predictive for at least six hours. BUN, pH, and Temp show intermediate horizons, where CVTA remains competitive for several hours before losing advantage against at least one reference model. CVTA's prediction skill for Creatinine, GCS, and HCT deteriorate earlier, suggesting that longer temporal context may improve predictions for these variables.

\section{CVTA Computational Efficiency}
\label{appendix:cvta_markov_state_computation}
\subsection{Efficient CVTA State Computation}
Below we provide efficient, parallelizable pseudocode for computing the CVTA state when the full observation sequence is available in a batched setting.

\begin{algorithm}[H]
\caption{Parallel construction of the CVTA state $\mathcal{S}_{t,f} = (x_{t,f}, m_{t,f}, \delta_{t,f}, \dot{x}_{t,f}, g_{t,f})$}
\label{alg:markov_state_parallel}
\begin{algorithmic}[1]
\REQUIRE Observation values $\mathbf{X} \in \mathbb{R}^{B \times T \times F}$, mask $\mathbf{M} \in \{0,1\}^{B \times T \times F}$, timestamps $\mathbf{T'} \in \mathbb{R}^{B \times T}$
\ENSURE CVTA state components for all $t \in [1,T]$, $f \in [1,F]$
\STATE \textbf{Initialize} cumulative observation count: $\mathbf{K} \leftarrow \text{cumsum}(\mathbf{M}, \text{axis}=1)$
\STATE \textbf{Observation mask}: $m_{t,f} \leftarrow \mathbb{I}[K_{t,f} > 0]$
\STATE \textbf{Slope mask}: $g_{t,f} \leftarrow \mathbb{I}[K_{t,f} \ge 2]$
\STATE \textbf{Last observation index}: $\tau_1 \leftarrow \text{argmax}_{\tau \le t} \{M_{\tau,f} = 1\}$ (vectorized)
\STATE \textbf{Second‑last observation index}: $\tau_2 \leftarrow \text{argmax}_{\tau < \tau_1} \{M_{\tau,f} = 1\}$ 
\STATE \textbf{Last value}: $x_{t,f} \leftarrow \mathbf{X}_{\tau_1,f}$
\STATE \textbf{Elapsed Time}: $\delta_{t,f} \leftarrow \mathbf{T'}_{t} - \mathbf{T'}_{\tau_1}$
\STATE \textbf{Local slope}: $\dot{x}_{t,f} \leftarrow g_{t,f} \cdot \dfrac{\mathbf{X}_{\tau_1,f} - \mathbf{X}_{\tau_2,f}}{\mathbf{T'}_{\tau_1} - \mathbf{T'}_{\tau_2}}$
\STATE \textbf{Return} $\mathcal{S}_{t,f} = (x_{t,f}, m_{t,f}, \delta_{t,f}, \dot{x}_{t,f}, g_{t,f})$
\end{algorithmic}
\end{algorithm}

\subsection{Online Inference}
To evaluate models in a realistic deployment scenario, we adopt the online inference protocol described by \cite{GRUwE} for simulating an ICU early warning system. In this setup, models process a continuous stream of retrospective EHR data in a stepwise fashion, mimicking real-time clinical monitoring. We replicate their experimental conditions using the Physionet dataset, streaming observations from 100 randomly selected patients and measuring total inference time and peak GPU memory usage during inference. Following their method, we compare CVTA against several strong baselines that excel at next‑observation prediction: mTAND, GraFITi, T‑PatchGNN, and HyperIMTS. In our experiments, we find that models like mTAND, GraFITi, T‑PatchGNN and HyperIMTS must buffer the entire observation history during inference, leading to time and memory costs that grow with sequence length. In contrast, models with a compact fixed-size state representation (e.g., GRUwE, CVTA) maintain constant inference cost independent of history length, making them substantially more efficient for online deployment, evident from \cref{fig:online_computational_cost}.
\begin{figure}[H]
    \centering
    \includegraphics[width=0.4\linewidth]{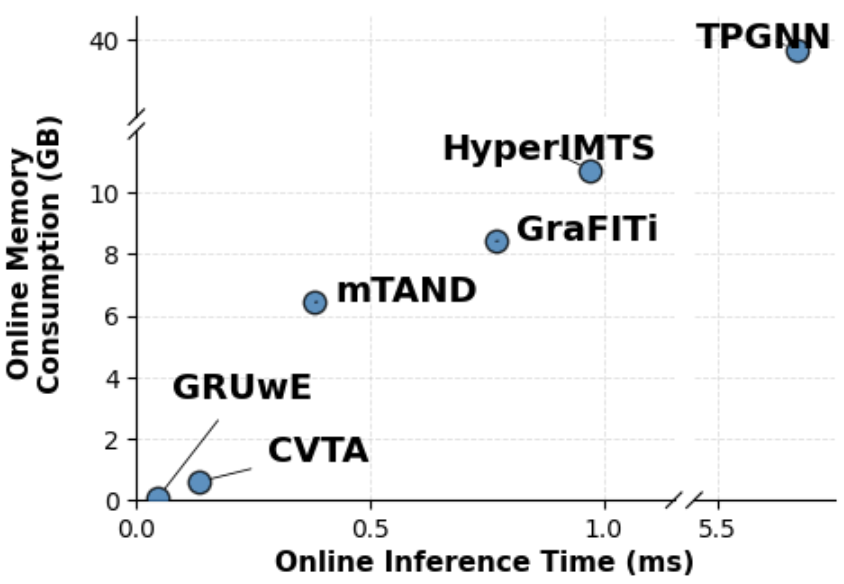}
    \caption{Efficiency of online inference. Scatter plot of inference latency versus GPU memory usage. Models without compact fixed-size state (top‑right) require buffering past observations, leading to higher computational costs, while models that support compact fixed-size states (bottom‑left) maintain constant, low overhead.}
    \label{fig:online_computational_cost}
\end{figure}

\section{Datasets}
\label{appendix:datasets}
\subsection{USHCN}
\noindent\textbf{URL.} \url{https://data.ess-dive.lbl.gov/view/doi%3A10.3334%2FCDIAC%2FCLI.NDP019}

\textbf{Description.} The United States Historical Climatology Network (USHCN) \cite{USHCNDataset} provides daily measurements of five meteorological variables---minimum temperature, maximum temperature, precipitation, snowfall, and snow depth---recorded at 1,218 stations across the United States. Multiple derived versions of this dataset have been used to benchmark irregularly sampled time series models.

\noindent\textbf{Pre‑processing.} This is the only Missing Completely At Random (MCAR) dataset in our benchmark. Following the protocol of \citet{CRUPaper}, we first select 1,168 stations over the four‑year period 1990--1993. To increase irregularity and sparsity, we subsample 50\% of the time points and mask an additional 20\% of the remaining observations. The data are split into 80\% for training/validation and 20\% for testing; from the training portion, 25\% is held out for validation.

\subsection{Physionet}
\noindent\textbf{URL.} \url{https://physionet.org/files/challenge-2012/1.0.0/}

\textbf{Description.} The PhysioNet/Computing in Cardiology Challenge 2012 \cite{PhysionetDataset} contains 8,000 ICU patient stays spanning 48 hours, each with 37 clinical time‑series variables recorded at irregular intervals. Variables include Non‑Invasive Mean Arterial Pressure, Platelets, Sodium, and numerous other vital signs and lab measurements.

\noindent\textbf{Pre‑processing.} We adopt the pre‑processing described by \citet{CRUPaper}: observation times are quantized to 6‑minute intervals. The data are split into 80\% for training/validation and 20\% for testing, with 25\% of the training portion reserved for validation.

\subsection{MIMIC-III \& MIMIC-HiRes}
\noindent\textbf{URL.} \url{https://physionet.org/content/mimiciii/1.4/}

\textbf{Description.} We extract the MIMIC‑III‑HiRes cohort from the MIMIC‑III database \cite{MIMICIIIDataset}. The cohort comprises 10,265 hospital admissions (8,799 unique patients) selected by: (i) availability in the MetaVision critical care system, (ii) admission length between 2 and 20 days, and (iii) patient age between 18 and 90 years.

\textbf{MIMIC-III Pre-processing.} Following the pre-processing routines described in \cite{GRUODEBayes}, we collect 96 univariate irregular time series from laboratory results and medication records. Each patient record spans 48 hours, observations in a 30 minutes time span are aggregated and constitute one sample sequence for training and evaluation. Models consume first 36 hours as context to make predictions at future time points. The data are split by patient, 90\% of admissions for training and validation and the remaining 10\% for testing.

\textbf{MIMIC-HiRes Pre-processing.} We follow the pre-processing pipeline of \cite{GRUwE} to extract 393 univariate, irregularly sampled time series from vital signs, laboratory measurements, medications, and procedures. No temporal rounding or resampling is applied, resulting in the highest degree of sparsity among all datasets evaluated in our benchmark. The data are split at the patient level, with 80\% of admissions used for training and validation (of which 20\% are reserved for validation) and the remaining 20\% held out for testing.

\subsection{MIMIC-IV}
\textbf{URL.} \url{https://physionet.org/content/mimiciv/3.1/}

\textbf{Pre-processing.} We use the publicly available MIMIC-IV dataset \cite{MIMIC_IV} and follow the data extraction and preprocessing pipeline of \cite{GRUODEBayes}. Specifically, we construct 100 univariate, irregularly sampled time series per patient admission. All observations are rounded to a 1-minute resolution, resulting in longer sequences than those in MIMIC-III for the same 48-hour observation window. For each admission, models are provided with the first 36 hours of data and tasked with predicting future time steps. The dataset is split at the patient level, with 90\% of records used for training and validation and the remaining 10\% reserved for testing.

\subsection{Dataset Statistics Definition}
We report the dataset statistics in \cref{tab:dataset_descriptive_stats}. To clarify how they are computed across all the datasets, we provide their definition below: 
\label{appendix:dataset_stats}
\begin{itemize}[noitemsep,topsep=0pt,partopsep=0pt] 
    \item \textbf{Missingness type:} MCAR (Missing Completely At Random), where the missingness is synthetically introduced by masking a subset of the observation and NMAR (Not Missing At Random), where the missingness is inherent in the data, were introduced in \cite{Singh1997Missingness} and \cite{Ramoni1997Missingness}.
    \item \textbf{Missingness (\%)}: For each sample, pad missing variables across all variables by using total time steps. 
    \begin{equation}
        missingness^{(i)} = \frac{\# \mathrm{\ of \  missing \ obs}^{(i)}}{\# \mathrm{\ of \ total \,obs}^{(i)}}.
    \end{equation}
    \item \textbf{Interarrival Variability:} To quantify the regularity of sampling in an irregular time series, we compute the coefficient of variation (CV) of the inter-observation intervals. 
Given a sequence of observation times $\{t_i\}_{i=1}^{N}$ with inter-arrival times $\Delta t_i = t_{i+1} - t_i$, the CV is defined as
\begin{equation}
\mathrm{CV}_{\Delta t} = \frac{\sqrt{\mathrm{Var}(\Delta t)}}{\mathbb{E}[\Delta t]}.
\end{equation}
The coefficient of variation provides a scale-invariant measure of dispersion of the sampling intervals. 
A low CV indicates nearly uniform sampling with consistent inter-observation gaps, while higher values correspond to increasingly irregular and bursty sampling patterns.
\end{itemize}

\section{Hyperparameter Search}
\label{appendix:hyperparameter_search}
In our experimental evaluation, we tune hyperparameters using grid search over predefined ranges for each model. Unless otherwise noted, we set the batch size to 50 for USHCN, 100 for Physionet, 100 for MIMIC-III, 32 for MIMIC-IV, and 1 for MIMIC-HiRes. When GPU memory errors are encountered, we reduce the batch size accordingly; if a model cannot be run even with the smallest feasible configuration, we report the result as out of memory in the main result tables. We train all models for 100 epochs on all datasets except MIMIC-HiRes, for which we train for 20 epochs. For each model and dataset, we select the configuration that achieves the best validation mean squared error and report its corresponding test-set results. Across all models, we vary the learning rate in $\{0.1, 0.01, 0.005, 0.001\}$, use Adam optimizer for all models, use a decay rate of $0.99$, and enable gradient clipping.

\begin{table}[H]
\centering
\small
\caption{Hyperparameter values for the grid search for all the methods.}
\begin{tabular}{lccccc}
\toprule
\textbf{Model} 
& \textbf{Latent Dim.} 
& \textbf{\# Layers} 
& \textbf{\# Heads} 
& \textbf{Embed. Dim.} 
& \textbf{Refs / Patches}  \\
\midrule
mTAND 
& 16, 32, 64, 128
& 1 
& 1, 2, 4 
& Time: 16, 32, 64 
& Ref pts: 32, 64
 \\

GRU-D 
& 16, 32, 64, 128 
& 1 
& -- 
& -- 
& -- 
 \\

CRU 
& 16, 32, 64, 128 
& 1 
& -- 
& -- 
& Basis: 10, 20 
 \\

f-CRU 
& 16, 32, 64, 128 
& 1
& -- 
& -- 
& Basis: 10, 20 
 \\

Latent ODE 
& 16, 32, 64, 128 
& Gen/Rec: 3 
& -- 
& Rec dim: 20, 32, 64 
& GRU units: 20 
 \\

ODE-RNN 
& 16, 32, 64, 128 
& 1 
& -- 
& -- 
& -- 
\\

ContiFormer 
& 16, 32, 64, 128 
& 1 
& -- 
& -- 
& -- 
 \\

GRU-$\Delta_t$ 
& 16, 32, 64, 128 
& 1 
& -- 
& -- 
& -- 
 \\

RKN-$\Delta_t$ 
& 16, 32, 64, 128 
& 1 
& -- 
& -- 
& -- 
 \\

T-PatchGNN 
& 16, 32, 64, 128 
& 1, 2, 4 
& 1, 2, 4 
& Time/Node: 8, 16, 32 
& Patches:  4, 8, 16
 \\

GraFITi 
& 16, 32, 64, 128 
& 1, 2, 4 
& 1, 2, 4 
& -- 
& -- 
 \\

GRUwE 
& 16, 32, 64, 128 
& 1 
& -- 
& -- 
& -- 
 \\

HyperIMTS 
& 16, 32, 64, 128 
& 1, 2, 4 
& 1, 2, 4 
& -- 
& -- 
\\

CVTA 
& 16, 32, 64, 128 
& 1 
& 1 
& -- 
& -- 
\\

\bottomrule
\end{tabular}
\label{tab:ushcn-hyperparameters}
\end{table}

\section{Ablation Studies}
\subsection{CVTA state component ablations}
\label{section:state_component_ablation}
To isolate the predictive contribution of each CVTA state component, we perform component-wise ablations in which each component ($x$, $\dot{x}$, $m$, $g$, $\delta$) is independently removed. For each ablation, we replace only the corresponding state input with an all-zero tensor before embedding, while leaving all other state components and the CVTA model architecture unchanged. These parameter-matched comparisons isolate the contribution of each state component without confounding the results with changes in model architecture or capacity. \cref{tab:state_component_ablation} reports the results on USHCN, Physionet, and MIMIC-III.

\begin{table}[H]
\centering
\scriptsize
\setlength{\tabcolsep}{1.5pt}
\caption{CVTA state component ablations. We report test MSE averaged over five random seeds and the change relative to CVTA as absolute $\Delta$MSE (relative \%).}
\label{tab:state_component_ablation}
\begin{tabular}{l|cc|cc|cc}
\toprule
\multirow{2}{*}{\textbf{Model}}
& \multicolumn{2}{c|}{\textbf{USHCN}}
& \multicolumn{2}{c|}{\textbf{Physionet}}
& \multicolumn{2}{c}{\textbf{MIMIC-III}} \\
& MSE$\scalebox{0.75}{$\times 10^{-2}$}$ & $\Delta$ (\%)
& MSE$\scalebox{0.75}{$\times 10^{-2}$}$ & $\Delta$ (\%)
& MSE$\scalebox{0.75}{$\times 10^{-1}$}$ & $\Delta$ (\%) \\
\midrule
\textbf{CVTA}
& $\best{0.003}$ & --
& $\best{0.226}$ & --
& $\best{3.957}$ & -- \\
\lightmidrule
w/o recent observation value ($x$)
& $0.010$ & $+0.007$ (233.3) & $0.249$ & $+0.023$ (10.2) & $4.921$ & $+0.964$ (24.4) \\
w/o local slope value ($\dot{x}$)
& $0.005$ & $+0.002$ (66.7) & $0.237$ & $+0.011$ (4.9) & $4.050$ & $+0.093$ (2.4) \\
w/o observation mask ($m$)
& $\best{0.003}$ & $+0.000$ (0.0) & $0.227$ & $+0.001$ (0.4) & $4.047$ & $+0.090$ (2.3) \\
w/o slope mask ($g$)
& $0.006$ & $+0.003$ (100.0) & $0.231$ & $+0.005$ (2.2) & $4.033$ & $+0.076$ (1.9) \\
w/o elapsed time ($\delta$)
& $0.004$ & $+0.001$ (33.3) & $0.241$ & $+0.015$ (6.6) & $4.112$ & $+0.155$ (3.9) \\
\bottomrule
\end{tabular}
\end{table}
The results in \cref{tab:state_component_ablation} suggest that removing the recent observation value causes the largest degradation on all three datasets, indicating that it is the most informative state component for CVTA's attention. On Physionet and MIMIC-III, removing any state component increases prediction error, with elapsed time being the next most important component, indicating that CVTA benefits from knowing how recently each variable was observed. Removing the local slope or masks causes smaller degradations than removing the recent observation value or elapsed time. Overall, the full CVTA state matches or outperforms every ablated variant across all three datasets.

\subsection{CVTA prediction function ablations}
\label{section:contribution_of_CVTA}
To isolate the effect of CVTA's attention mechanism, we compare it with non-neural, neural, and attention baselines that use the same CVTA state. 

\textbf{Non-neural baselines.} We consider two non-neural baselines that are deterministic functions of the CVTA state. 
\begin{itemize}[noitemsep,topsep=0pt,partopsep=0pt]
    \item \textbf{Simple Persistence} predicts each target variable using its most recent observation, $\hat{y}_{t+\Delta T,f} = x_{t,f}$, falling back on the training-set variable mean when no prior observation exists.
    \item \textbf{Linear Extrapolation} extends the most recent observation along the local slope between the two most recent observations, predicting
$\hat{y}_{t+\Delta T,f} = x_{t,f} + \dot{x}_{t,f}\,\delta'_{t,f}$, where
$\delta'_{t,f} = \delta_{t,f} + \Delta T$, and falls back on Simple Persistence
when fewer than two observations are available.
\end{itemize}

\textbf{Neural baselines.} We consider feed-forward network (FFN) based predictors, where each predictor takes the concatenated CVTA state components as input, forming a fixed $5F$-dimensional vector, where $F$ is the number of variables. We vary the FFN depth (\texttt{D}) and latent size (\texttt{lsd}) to control model capacity. We also test a residual variant (\texttt{res}=T), where the model output is added to the last observed values (i.e., \cref{eq:final_prediction}) to match CVTA's prediction function.

\textbf{Attention baselines.}
We consider two parameter-matched CVTA variants that each remove one property of the attention: cross-variable interactions or replace VCE with additive variable identity embeddings:
\begin{itemize}[noitemsep,topsep=0pt,partopsep=0pt]
    \item \textbf{w/o cross-variable interactions} preserves the CVTA architecture and number of parameters while performing diagonal attention in both the recent-observation and local-slope attention branches, i.e., for each target variable $f$, a fixed diagonal mask restricts the admissible key--value pairs to those of variable $f$, directly measuring the contribution of cross-variable interactions to CVTA's prediction function.
    \item \textbf{w/o variable-specific maps} is a parameter-matched variant that replaces the variable-specific affine maps $\phi^{(f)}_{\bullet}$ with shared maps and adds learned source variable embeddings to the keys and values, isolating if the gains of VCEs stem from encoding variable identity or from learning variable-specific affine transformations.
\end{itemize}

\begin{table}[t]
\centering
\scriptsize
\setlength{\tabcolsep}{1.5pt}
\caption{Comparison of CVTA with non-neural, neural, and attention baselines that use the same CVTA state. We report test MSE averaged over five random seeds and the change relative to CVTA as absolute $\Delta$MSE (relative \%).}
\label{tab:ffn_state_ablation}
\begin{tabular}{l|cc|cc|cc}
\toprule
\multirow{2}{*}{\textbf{Model}}
& \multicolumn{2}{c|}{\textbf{USHCN}}
& \multicolumn{2}{c|}{\textbf{Physionet}}
& \multicolumn{2}{c}{\textbf{MIMIC-III}} \\
& MSE$\scalebox{0.75}{$\times 10^{-2}$}$ & $\Delta$ (\%)
& MSE$\scalebox{0.75}{$\times 10^{-2}$}$ & $\Delta$ (\%)
& MSE$\scalebox{0.75}{$\times 10^{-1}$}$ & $\Delta$ (\%) \\
\midrule
\textbf{CVTA}
& $\best{0.003}$ & --
& $\best{0.226}$ & --
& $\best{3.957}$ & -- \\
\lightmidrule
\multicolumn{7}{l}{\textbf{\textit{Non-neural baselines}}} \\
Simple Persistence
& $0.012$ & $+0.009$ (300.0) & $0.273$ & $+0.047$ (20.8) & $5.687$ & $+1.730$ (43.7) \\
Linear Extrapolation
& $0.056$ & $+0.053$ (1766.7) & $2.539$ & $+2.313$ (1023.5) & $16.587$ & $+12.630$ (319.2) \\
\lightmidrule
\multicolumn{7}{l}{\textbf{\textit{Neural baselines}}} \\
FFN(\texttt{D}=2, \texttt{lsd}=256, \texttt{res}=F)
& $0.032$ & $+0.029$ (966.7) & $0.337$ & $+0.111$ (49.1) & $4.995$ & $+1.038$ (26.2) \\
FFN(\texttt{D}=2, \texttt{lsd}=64, \texttt{res}=T)
& $0.033$ & $+0.030$ (1000.0) & $0.295$ & $+0.069$ (30.5) & $4.752$ & $+0.795$ (20.1) \\
FFN(\texttt{D}=2, \texttt{lsd}=128, \texttt{res}=T)
& $0.024$ & $+0.021$ (700.0) & $0.271$ & $+0.045$ (19.9) & $4.631$ & $+0.674$ (17.0) \\
FFN(\texttt{D}=2, \texttt{lsd}=256, \texttt{res}=T)
& $0.023$ & $+0.020$ (666.7) & $0.269$ & $+0.043$ (19.0) & $4.626$ & $+0.669$ (16.9) \\
\lightmidrule
\multicolumn{7}{l}{\textbf{\textit{Attention baselines}}} \\
w/o cross-variable interactions
& $0.005$ & $+0.002$ (66.7) & $0.238$ & $+0.012$ (5.3) & $4.345$ & $+0.388$ (9.8) \\
w/o variable-specific maps
& $0.007$ & $+0.004$ (133.3) & $0.243$ & $+0.017$ (7.5) & $4.173$ & $+0.216$ (5.5) \\
\bottomrule
\end{tabular}
\end{table}
Based on the results in \cref{tab:ffn_state_ablation}, simple persistence is the stronger of the two non-neural baselines, while linear extrapolation degrades substantially, as extrapolating a noisy two-point slope amplifies errors over the prediction horizon. CVTA outperforms both across all three datasets, showing that its gains are not explained by carrying forward recent values or local trends alone. Residual FFNs generally improve over the non-residual variant, confirming the benefit of the residual prediction formulation in \cref{eq:final_prediction}. However, even higher-capacity FFNs consistently underperform CVTA across USHCN, Physionet, and MIMIC-III. Both attention baselines outperform all FFN baselines but underperform CVTA on all three datasets, indicating that CVTA's gains depend on attending across variables and on variable-specific transformations, not only on encoding variable identity. Overall, CVTA's gains come from combining its compact state with cross-variable temporal attention, not from the state alone.

\subsection{Other CVTA ablations}
\label{appendix:cvta_ablation_result}
While the previous ablations isolate what information the CVTA state provides and what the prediction function contributes, here we evaluate the mechanisms CVTA uses to address the core challenges of irregular multivariate time series prediction: variable heterogeneity, local temporal dynamics, missingness, and irregular timing. These ablations correspond to \Cref{tab:ablation_mse_drop_avg} in the main paper, and we report full MSE and MAE results across all five benchmarks in \cref{tab:ablation_mse_mae}.

\begin{itemize}[noitemsep,topsep=0pt,partopsep=0pt]
    \item \textbf{Variable heterogeneity (w/o VCE):} This variant replaces variable-conditioned embeddings with parameter-matched shared embeddings while keeping the proposed attention formulation unchanged.
    
    \item \textbf{Local temporal dynamics (w/o local slope):} This variant disables the local-slope branch and relies solely on the recent-observation branch for generating predictions.
    
    \item \textbf{Missingness in keys (w/o key mask):} In this variant, the observation and slope masks $\{\mathbf{m}_t, \mathbf{g}_t\}$ embeddings are omitted from the key embeddings in both attention branches, preventing the model from explicitly distinguishing between observed and missing entries at the key level.
    
    \item \textbf{Missingness in values (w/o value mask):} This variant removes the observation and slope masks $\{\mathbf{m}_t, \mathbf{g}_t\}$ embeddings from the value embeddings in both attention branches, thereby eliminating explicit masking information from the value representations.
    
    \item \textbf{Irregular timing (w/o elapsed time):} This variant sets the elapsed-time CVTA state component to zero for all variables, i.e., $\boldsymbol{\delta}_t = \mathbf{0}$, forcing the model to ignore timing information and rely only on value-based signals.
\end{itemize}

\begin{table}[H]
\centering
\tiny
\caption{CVTA architectural ablations (MSE / MAE) across datasets.}
\label{tab:ablation_mse_mae}
\setlength{\tabcolsep}{1.42pt}
\begin{tabular}{l|cccccccccc}
\toprule
\textbf{Model} 
& \multicolumn{2}{c}{\textbf{USHCN}}
& \multicolumn{2}{c}{\textbf{Physionet}}
& \multicolumn{2}{c}{\textbf{MIMIC-III}}
& \multicolumn{2}{c}{\textbf{MIMIC-IV}}
& \multicolumn{2}{c}{\textbf{MIMIC-HiRes}} \\
\cmidrule(lr){2-3}
\cmidrule(lr){4-5}
\cmidrule(lr){6-7}
\cmidrule(lr){8-9}
\cmidrule(lr){10-11}
&  MSE$\scalebox{0.75}{$\times 10^{-2}$}$ & MAE$\scalebox{0.75}{$\times 10^{-2}$}$
& MSE$\scalebox{0.75}{$\times 10^{-2}$}$ & MAE$\scalebox{0.75}{$\times 10^{-2}$}$
& MSE$\scalebox{0.75}{$\times 10^{-1}$}$ & MAE$\scalebox{0.75}{$\times 10^{-1}$}$
& MSE$\scalebox{0.75}{$\times 10^{-1}$}$ & MAE$\scalebox{0.75}{$\times 10^{-1}$}$
& MSE$\scalebox{0.75}{$\times 10^{-2}$}$ & MAE$\scalebox{0.75}{$\times 10^{-2}$}$ \\
\midrule
\textbf{CVTA}
& $\best{0.003}\pms{0.002}$ & $\best{0.112}\pms{0.033}$
& $\best{0.226}\pms{0.009}$ & $\best{2.543}\pms{0.024}$
& $\best{3.957}\pms{0.009}$ & $\best{3.778}\pms{0.016}$
& $\best{2.026}\pms{0.006}$ & $\best{2.811}\pms{0.016}$
& $\best{0.513}\pms{0.018}$ & $\best{3.941}\pms{0.034}$ \\
\lightmidrule
w/o VCE
& $0.013\pms{0.003}$ & $0.191\pms{0.022}$
& $0.494\pms{0.008}$ & $3.259\pms{0.036}$
& $5.236\pms{0.008}$ & $4.224\pms{0.023}$
& $2.932\pms{0.006}$ & $3.284\pms{0.019}$
& $1.016\pms{0.036}$ & $6.132\pms{0.061}$ \\
w/o local slope
& $0.006\pms{0.004}$ & $0.124\pms{0.031}$
& $0.241\pms{0.009}$ & $2.626\pms{0.032}$
& $4.005\pms{0.009}$ & $3.829\pms{0.021}$
& $2.129\pms{0.009}$ & $2.886\pms{0.021}$
& $0.903\pms{0.018}$ & $4.606\pms{0.031}$ \\
w/o key mask
& $0.004\pms{0.003}$ & $0.114\pms{0.032}$
& $0.227\pms{0.010}$ & $2.563\pms{0.034}$
& $3.960\pms{0.009}$ & $3.791\pms{0.018}$
& $2.080\pms{0.006}$ & $2.838\pms{0.015}$
& $0.809\pms{0.016}$ & $4.541\pms{0.032}$ \\
w/o value mask
& $0.010\pms{0.001}$ & $0.213\pms{0.031}$
& $0.237\pms{0.007}$ & $2.595\pms{0.021}$
& $4.110\pms{0.007}$ & $3.875\pms{0.023}$
& $2.117\pms{0.006}$ & $2.888\pms{0.014}$
& $0.832\pms{0.018}$ & $4.688\pms{0.030}$ \\
w/o elapsed time
& $0.004\pms{0.003}$ & $0.138\pms{0.036}$
& $0.241\pms{0.004}$ & $2.658\pms{0.019}$
& $4.112\pms{0.008}$ & $3.848\pms{0.015}$
& $2.168\pms{0.005}$ & $2.940\pms{0.015}$
& $0.869\pms{0.019}$ & $4.712\pms{0.028}$ \\
\bottomrule
\end{tabular}
\end{table}

\section{Integration of CVTA with other prediction models}
\label{section:cvta_base_model_integration}

CVTA can be integrated as an auxiliary module with existing irregular time series models. In this section, we present implementation details for using CVTA as an auxiliary module with existing irregular time series models: GraFITi, GRUwE, and HyperIMTS. As outlined in \cref{section:cvta_as_aux_pred_head}, the key step is obtaining the latent representation from each base model $\mathbf{h}_{t+\Delta T}^{{\scriptscriptstyle \mathrm{base}}}$, which subsequently drives the meta-decision model. Below, we specify how this latent representation is derived for each model variant.
\subsection{\textsc{GraFITi+CVTA}}
To model irregular multivariate time series, GraFITi \cite{grafitiPaper} builds a sparsity structured graph in the form of a bipartite graph connecting time points and variables nodes, where edges encode observation mask and observation value. This graph representation enables GraFITi to explicitly model irregular sampling patterns and missing data by performing message passing between time and variable nodes. We summarize the GraFITi representations that form the latent representation:
\begin{itemize}
    \item $\mathbf{h}^{\text{node},\,l}_u$ denotes the representation of a variable node $u$ at GNN layer $l$, capturing the latent state of the corresponding variable after $l$ rounds of message passing.
    
    \item $\mathbf{h}^{\text{node},\,l}_v$ denotes the representation of a time (or query-time) node $v$ at GNN layer $l$, encoding temporal context at the associated timestamp.
    
    \item $\mathbf{h}^{\text{edge},\,l}_{e}$ denotes the representation of an edge $e = \{u, v\}$ connecting variable node $u$ and time node $v$ at layer $l$, which models the observed value (or target query) and its interaction with the variable and time representations.
\end{itemize}
Let $L$ be the number of stacked graph neural network layers used to iteratively update node and edge embeddings before decoding the final predictions. For \textsc{GraFITi+CVTA}, to extract the latent representation for the meta-decision model, we concatenate the representations of variable node ($\mathbf{h}^{\text{node},\,L}_u$), time ($\mathbf{h}^{\text{node},\,L}_v$) and the edges ($\mathbf{h}^{\text{edge},\,L}_{e}$) at the last layer of the stack:
\begin{equation}
    \mathbf{h}_{t+\Delta T}^{{\scriptscriptstyle \mathrm{base}}} = \mathrm{Concat}\Bigl(\Bigl[\mathbf{h}_u^{\mathrm{node},\, L} \,; \, \mathbf{h}_v^{\mathrm{node},\, L} \, ; \, \mathbf{h}_e^{\mathrm{edge}, \,L}\Bigr] \Bigr).
\end{equation}

\subsection{\textsc{GRUwE+CVTA}}
GRUwE \cite{GRUwE} is an auto-regressive model that maintains an explicit latent state to summarize the historical information of an irregularly sampled multivariate time series. The model updates its latent state through two complementary mechanisms: a time-triggered reset, which accounts for temporal decay between observations, and an observation-triggered reset, which incorporates newly observed values. Predictions are generated by linearly projecting a time-decayed version of the latent state. Below, we briefly review the GRUwE representations that are fed as input to the meta-decision model:
\begin{itemize}
    \item $\mathbf{h}_t$ denotes the latent state representation of the multivariate time series at time $t$, summarizing all information from past observations up to time step $t$.

    \item $\mathbf{g}_{t,\Delta T}$ denotes the time-decayed latent state representation at a future prediction horizon $\Delta T$, obtained by applying learnable exponential decay functions to $\mathbf{h}_t$ (i.e., $\mathbf{g}_{t,\Delta T} = \mathbf{h}_t \odot \boldsymbol{\gamma}({\Delta T})$). It models the continuous-time evolution of the latent state from $t$ to $t + \Delta T$.
\end{itemize}
Since $\mathbf{g}_{t,\Delta T}$ already contains the complete latent state necessary for prediction, we directly use it to parameterize the meta-decision function. Thus, for \textsc{GRUwE+CVTA} we define:

\begin{equation}
    \mathbf{h}_{t+\Delta T}^{{\scriptscriptstyle \mathrm{base}}} = \mathbf{g}_{t,\Delta T}.
\end{equation}

\subsection{\textsc{HyperIMTS+CVTA}}
HyperIMTS \cite{HyperIMTS} formulates irregular multivariate time series forecasting as a node prediction problem on a hypergraph, where observed values are represented as nodes and temporal and variable dependencies are captured through hyperedge-based message passing. By using irregularity-aware hyperedge-to-hyperedge and hyperedge-to-node updates, HyperIMTS produces context-aware node representations that encode both temporal and cross-variable information. Below, we summarize the HyperIMTS representations that are fed to the meta-decision model:
\begin{itemize}
    \item $\mathbf{V}''$ denotes the node representation of the variables at future time point obtained after $L$ layers of hypergraph message passing.
    \item $\mathbf{E}'_{\mathrm{time}}$ denotes the aggregated latent representation of temporal hyperedges, capturing global temporal dependencies across irregular time points.
    \item $\mathbf{E}''_{\mathrm{var}}$ denotes the aggregated latent representation of variable hyperedges, encoding cross-variable interactions.
\end{itemize}

Similar to the HyperIMTS's prediction function (Eq.~(15) of the original paper \cite{HyperIMTS}), we concatenate the node and aggregated hyperedge representations to feed it to the meta-decision model of the \textsc{HyperIMTS+CVTA}:
\begin{equation}
    \mathbf{h}_{t+\Delta T}^{{\scriptscriptstyle \mathrm{base}}}
    = \mathrm{Concat} \Bigl( \Bigl[ \mathbf{V}'' \,; \, \mathbf{E}'_{\mathrm{time}}\,; \, \mathbf{E}''_{\mathrm{var}} \Bigl] \Bigr).
\end{equation}

\subsection{Results on Multi-Step Prediction}

\textbf{Parameter-Matched Controls.}
For each base model, we create a parameter-matched variant, denoted `Base (PM)', by increasing its latent dimension until it has at least as many parameters as the corresponding Base+CVTA model. The base architecture and training setup are otherwise unchanged.

Along with CVTA, Base+CVTA, we train Base (PM) models for the base models: GRUwE, GraFITi, and HyperIMTS and report results for prediction horizons $H=\{1,3,5,7\}$ across all benchmark datasets in \cref{tab:results_transposed_full}.

\begin{table}[H]
\tiny
\centering
\setlength{\tabcolsep}{1.pt}
\caption{Multi-step prediction performance in MSE and MAE across all datasets for horizons $H=\{1,3,5,7\}$. Results report mean and standard deviation over five seeds. `--' indicates GPU out-of-memory error during training.}
\label{tab:results_transposed_full}
\begin{tabular}{l c cccc cccc cccc cccc cccc}
\toprule
\textbf{Model} & \textbf{Metric}
& \multicolumn{4}{c}{\textbf{USHCN}}
& \multicolumn{4}{c}{\textbf{Physionet}}
& \multicolumn{4}{c}{\textbf{MIMIC-III}}
& \multicolumn{4}{c}{\textbf{MIMIC-IV}}
& \multicolumn{4}{c}{\textbf{MIMIC-HiRes}} \\
\cmidrule(lr){3-6}\cmidrule(lr){7-10}\cmidrule(lr){11-14}\cmidrule(lr){15-18}\cmidrule(lr){19-22}
& & H=1 & H=3 & H=5 & H=7
  & H=1 & H=3 & H=5 & H=7
  & H=1 & H=3 & H=5 & H=7
  & H=1 & H=3 & H=5 & H=7
  & H=1 & H=3 & H=5 & H=7 \\
\midrule

\textbf{CVTA} & Mean MSE &
0.003 & 0.010 & 0.013 & 0.015 &
0.226 & 0.245 & 0.254 & 0.274 &
3.957 & 4.155 & 4.282 & 4.435 &
2.026 & 2.259 & 2.382 & 2.452 &
0.513 & 0.832 & 0.913 & 1.041 \\
& Std MSE &
0.002 & 0.002 & 0.004 & 0.007 &
0.009 & 0.014 & 0.012 & 0.021 &
0.009 & 0.008 & 0.014 & 0.013 &
0.006 & 0.007 & 0.011 & 0.023 &
0.018 & 0.020 & 0.021 & 0.034 \\
\lightmidrule
\textbf{GRUwE (PM)} & Mean MSE &
0.011 & 0.033 & 0.118 & 0.184 &
0.234&0.250&0.262&0.279&
4.010 & 4.113 & 4.216 & 4.322 &
2.110 & 2.287 & 2.376 & 2.461 &
0.809 & 1.224 & 1.326 & 1.329 \\
& Std MSE &
0.014 & 0.015 & 0.016 & 0.021 &
0.007 & 0.004 & 0.008 & 0.015 &
0.008 & 0.008 & 0.013 & 0.012 &
0.008 & 0.013 & 0.011 & 0.016 &
0.022 & 0.021 & 0.032 & 0.030 \\

\textbf{\textsc{GRUwE+CVTA}} & Mean MSE &
0.004 & 0.017 & 0.024 & 0.028 &
0.211 & 0.230 & 0.239 & 0.258 &
3.783 & 3.925 & 4.039 & 4.151 &
2.015 & 2.191 & 2.285 & 2.334 &
0.604 & 0.931 & 0.959 & 0.989 \\
& Std MSE &
0.014 & 0.011 & 0.018 & 0.024 &
0.008 & 0.006 & 0.016 & 0.021 &
0.006 & 0.012 & 0.009 & 0.020 &
0.010 & 0.015 & 0.011 & 0.029 &
0.026 & 0.031 & 0.032 & 0.032 \\
\lightmidrule
\textbf{GraFITi (PM)} & Mean MSE &
0.072 & 0.081& 0.099&0.113 &
0.234 & 0.243 & 0.247 & 0.256 &
4.059 & 4.083 & 4.216 & 4.264 &
2.146 & 2.277 & 2.370 & 2.371 &
 1.409 & 1.445 & 1.462 & 1.703 \\
& Std MSE &
0.014 & 0.015 & 0.021 & 0.022 &
0.013 & 0.021 & 0.018 & 0.023 &
0.015 & 0.013 & 0.011 & 0.015 &
0.014 & 0.012 & 0.018 & 0.022 &
0.031 & 0.031 & 0.034 & 0.034 \\

\textbf{\textsc{GraFITi+CVTA}} & Mean MSE &
0.013 & 0.020 & 0.042 & 0.044 &
0.225 & 0.235 & 0.235 & 0.248 &
3.960 & 3.962 & 4.029 & 4.111 &
2.037 & 2.209 & 2.264 & 2.269 &
0.769 & 1.001 & 1.039 & 1.248 \\
& Std MSE &
 0.011 & 0.015 & 0.029 & 0.022 &
0.016 & 0.014 & 0.011 & 0.021 &
0.009 & 0.013 & 0.011 & 0.019 &
0.012 & 0.014 & 0.023 & 0.030 &
0.029 & 0.035 & 0.036 & 0.039 \\
\lightmidrule
\textbf{HyperIMTS (PM)} & Mean MSE &
0.032 &	0.034 &	0.045& 0.046 &
0.231 & 0.242 & 0.247 & 0.253 &
3.817 & 3.886 & 3.913 & 3.942 &
2.054 & 2.142 & 2.208 & 2.266 &
-- & -- & -- & -- \\
& Std MSE &
0.013 & 0.017 & 0.016 & 0.021 &
0.003 & 0.004 & 0.012 & 0.009 &
0.018 & 0.018 & 0.025 & 0.026 &
0.015 & 0.014 & 0.016 & 0.019 &
-- & -- & -- & -- \\

\textbf{\textsc{HyperIMTS+CVTA}} & Mean MSE &
0.010 & 0.011 & 0.013 & 0.016 &
0.226 & 0.234 & 0.235 & 0.247 &
3.755 & 3.759 & 3.816 & 3.939 &
1.978 & 2.119 & 2.179 & 2.215 &
-- & -- & -- & -- \\
& Std MSE &
0.011 & 0.021 & 0.017 & 0.026 &
0.007 & 0.004 & 0.007 & 0.014 &
0.018 & 0.016 & 0.021 & 0.024 &
0.012 & 0.013 & 0.026 & 0.030 &
-- & -- & -- & -- \\

\midrule
\addlinespace[0.75em]
\textbf{CVTA} & Mean MAE &
0.112 & 0.229 & 0.262 & 0.287 &
2.543 & 2.772 & 2.847 & 2.954 &
3.778 & 3.956 & 4.017 & 4.115 &
2.811 & 2.867 & 2.843 & 2.861 &
3.941 & 5.157 & 5.662 & 5.819 \\
& Std MAE &
0.033 & 0.031 & 0.034 & 0.039 &
0.024 & 0.021 & 0.026 & 0.033 &
0.016 & 0.017 & 0.020 & 0.025 &
0.016 & 0.022 & 0.023 & 0.031 &
0.034 & 0.032 & 0.041 & 0.039 \\
\lightmidrule
\textbf{GRUwE (PM)} & Mean MAE &
0.241 & 0.412 & 0.819 & 1.035 &
2.832 & 3.019 & 3.115 & 3.212 &
3.911 & 4.037 & 4.098 & 4.167 &
2.971 & 3.006 & 3.066 & 3.076 &
5.274 & 6.912 & 7.029 & 7.077 \\
& Std MAE &
0.185 & 0.122 & 0.143 & 0.196 &
0.034 & 0.031 & 0.043 & 0.052 &
0.022 & 0.026 & 0.029 & 0.033 &
0.063 & 0.069 & 0.071 & 0.068 &
0.037 & 0.034 & 0.039 & 0.041 \\

\textbf{\textsc{GRUwE+CVTA}} & Mean MAE &
0.182 & 0.351 & 0.358 & 0.418 &
2.612 & 2.779 & 2.877 & 2.991 &
3.775 & 3.858 & 3.932 & 4.016 &
2.881 & 2.879 & 2.898 & 2.888 &
4.390 & 5.744 & 5.857 & 5.979 \\
& Std MAE &
0.187 & 0.213 & 0.216 & 0.201 &
0.031 & 0.039 & 0.044 & 0.053 &
0.025 & 0.028 & 0.031 & 0.038 &
0.055 & 0.062 & 0.065 & 0.069 &
0.044 & 0.039 & 0.051 & 0.055 \\
\lightmidrule
\textbf{GraFITi (PM)} & Mean MAE &
0.672 & 0.752 & 0.806 & 0.884 &
2.683 & 2.805 & 2.876 & 2.919 &
3.825 & 3.939 & 4.041 & 4.074 &
2.984 & 2.976 & 2.939 & 2.883 &
7.441 & 7.678 & 8.343 & 8.406 \\
& Std MAE &
0.043 & 0.039 & 0.048 & 0.053 &
0.024 & 0.029 & 0.041 & 0.032 &
0.113 & 0.136 & 0.108 & 0.122 &
0.057 & 0.056 & 0.067 & 0.063 &
0.126 & 0.214 & 0.154 & 0.132 \\

\textbf{\textsc{GraFITi+CVTA}} & Mean MAE &
0.318 & 0.341 & 0.486 & 0.488 &
2.659 & 2.784 & 2.825 & 2.902 &
3.746 & 3.847 & 3.885 & 3.948 &
2.894 & 2.865 & 2.874 & 2.773 &
4.906 & 6.002 & 6.315 & 6.376 \\
& Std MAE &
0.039 & 0.045 & 0.046 & 0.044 &
0.021 & 0.033 & 0.029 & 0.036 &
0.126 & 0.152 & 0.158 & 0.165 &
0.055 & 0.051 & 0.062 & 0.068 &
0.114 & 0.164 & 0.132 & 0.149 \\
\lightmidrule
\textbf{HyperIMTS (PM)} & Mean MAE &
0.323 & 0.438 & 0.498 & 0.510 &
2.716 & 2.757 & 2.773 & 2.844 &
3.719 & 3.741 & 3.815 & 3.891 &
2.737 & 2.782 & 2.806 & 2.808 &
-- & -- & -- & -- \\
& Std MAE &
0.020 & 0.040 & 0.036 & 0.045 &
0.023 & 0.031 & 0.032 & 0.038 &
0.132 & 0.116 & 0.164 & 0.173 &
0.024 & 0.019 & 0.043 & 0.048 &
-- & -- & -- & -- \\

\textbf{\textsc{HyperIMTS+CVTA}} & Mean MAE &
0.184 & 0.189 & 0.204 & 0.280 &
2.619 & 2.713 & 2.771 & 2.842 &
3.652 & 3.721 & 3.814 & 3.882 &
2.731 & 2.740 & 2.747 & 2.749 &
-- & -- & -- & -- \\
& Std MAE &
0.019 & 0.018 & 0.021 & 0.030 &
0.022 & 0.027 & 0.032 & 0.038 &
0.129 & 0.144 & 0.140 & 0.168 &
0.033 & 0.042 & 0.051 & 0.046 &
-- & -- & -- & -- \\

\bottomrule
\end{tabular}
\end{table}

\subsection{Meta-Decision Model Behavior}
\label{appendix:meta_decision}
To analyze how the meta-decision model adapts across the forecast horizon, we train the GRUwE+CVTA and GraFITi+CVTA models on longer future horizons ($H=7$) and examine how their reliance on CVTA evolves over time. Specifically, we partition the prediction horizon into temporal bins and compute the mean and standard deviation of the meta-decision weights assigned to CVTA, within each bin. Because these weights determine CVTA's contribution to the final prediction, higher values indicate greater reliance on CVTA. For Physionet, where both models observe the first 24 hours, we report the CVTA weights for each future temporal bin in \cref{tab:meta_decision_weights}, together with the numbers of samples and prediction targets represented in that bin within the held-out set.

\begin{table}[H]
\centering
\small
\caption{Meta-decision weight assigned to CVTA ($1-\alpha_{t+\Delta T}$) across prediction horizons on Physionet (mean $\pm$ std).}
\label{tab:meta_decision_weights}
\begin{tabular}{ccccc}
\toprule
\textbf{Bin center (h)} & \textbf{Target count} & \textbf{Samples} & \textbf{GraFITi+CVTA} & \textbf{GRUwE+CVTA} \\
\midrule
24.5 & 13{,}632 & 1{,}513 & $0.50 \pm 0.08$ & $0.32 \pm 0.04$ \\
25.5 & 12{,}666 & 1{,}486 & $0.45 \pm 0.09$ & $0.31 \pm 0.04$ \\
26.5 & 12{,}002 & 1{,}461 & $0.41 \pm 0.10$ & $0.31 \pm 0.04$ \\
27.5 & 10{,}689 & 1{,}350 & $0.44 \pm 0.11$ & $0.31 \pm 0.04$ \\
28.4 &  8{,}630 & 1{,}131 & $0.43 \pm 0.11$ & $0.31 \pm 0.04$ \\
29.4 &  5{,}806 &   809 & $0.43 \pm 0.10$ & $0.30 \pm 0.04$ \\
30.4 &  2{,}787 &   425 & $0.43 \pm 0.12$ & $0.30 \pm 0.04$ \\
31.4 &    625 &    92 & $0.40 \pm 0.14$ & $0.30 \pm 0.04$ \\
32.4 &    346 &    49 & $0.41 \pm 0.14$ & $0.30 \pm 0.04$ \\
33.4 &    154 &    26 & $0.43 \pm 0.19$ & $0.29 \pm 0.04$ \\
34.4 &    210 &    29 & $0.36 \pm 0.13$ & $0.29 \pm 0.03$ \\
\bottomrule
\end{tabular}
\end{table}

Across both models, CVTA receives substantial weight at every horizon (approximately 0.30--0.50), showing that the meta-decision model consistently combines CVTA with the base model rather than relying on either predictor alone. In both GraFITi+CVTA and GRUwE+CVTA, the weight assigned to CVTA decreases overall with forecasting horizon. Over approximately 10 hours (24.5--34.4 hours), the mean CVTA weight decreases from 0.50 to 0.36 for GraFITi+CVTA (a 28\% relative reduction), and from 0.32 to 0.29 for GRUwE+CVTA (a 9.4\% reduction). This pattern shows that the meta-decision model does not rely on CVTA uniformly: it favors CVTA near the beginning of the forecasting window and gradually places greater weight on the base model at longer horizons. The stronger decline for GraFITi+CVTA further indicates that this temporal adaptation depends on the complementary strengths of the base model.

\section{Interpretability Case Studies}
\label{section:clinical_case_studies}

\begin{table}[H]
\centering
\small
\caption{Top-8 attended variables for high-lactate and high-creatinine clinical scenarios.}
\label{tab:clinical_case_attention}
\setlength{\tabcolsep}{4pt}
\begin{tabular}{l p{0.78\linewidth}}
\toprule
\textbf{Target} & \textbf{Top-8 attended variables with attention scores} \\
\midrule
\textbf{Lactate} &
HCO$_3$ (0.2208), Lactate (0.1651), RespRate (0.1095), TroponinT (0.0771), Platelets (0.0636), ALT (0.0589), PaCO$_2$ (0.0376), FiO$_2$ (0.0357) \\

\textbf{Creatinine} &
Creatinine (0.8387), TroponinT (0.0392), Urine (0.0317), BUN (0.0146), Cholesterol (0.0133), Bilirubin (0.0112), HCO$_3$ (0.0110), ALP (0.0076) \\
\bottomrule
\end{tabular}
\end{table}
\textbf{Interpreting CVTA in a High-Lactate Clinical Scenario.} We select a prediction time point where the lactate level is clinically high ($\geq$ 12.4 mmol/L), indicating a critical patient state. Lactate is a marker of the balance between tissue oxygen delivery and utilization, systemic stress, and metabolic clearance. It is therefore connected to variables that reflect three broad processes: acid–base compensation, perfusion and organ injury, and hepatic/metabolic clearance. In this example, lactate rises from 16.96 (observed) to 22.42 mmol/L, while CVTA predicts 21.33 mmol/L, decomposed into +4.13 from the recent observation and +0.24 from the local slope branch. For this specific instance, we report CVTA's top-8 attended variables in \cref{tab:clinical_case_attention}.

HCO3 (0.2208) is linked to lactate through metabolic acidosis / acid–base balance, since lactate accumulation and bicarbonate buffering are tightly coupled. Lactate itself (0.1651) provides direct short-term continuity of the target signal. RespRate (0.1095) and PaCO2 (0.0376) relate to lactate through respiratory compensation, as ventilation adjusts to metabolic acidosis. TroponinT (0.0771) reflects cardiac injury and impaired perfusion, which can contribute to lactate generation through tissue hypoxia. Platelets (0.0636) relate through microcirculatory dysfunction and severe systemic illness, which may reduce effective oxygen delivery. ALT (0.0589) is connected through hepatic function and lactate clearance, since the liver is a major site of lactate metabolism. FiO2 (0.0357) reflects oxygen support and concern for impaired oxygenation, which is clinically tied to the same hypoxia/perfusion processes that influence lactate.

\textbf{Interpreting CVTA in a High-Creatinine Clinical Scenario.} We analyze a prediction time point where the creatinine level is abnormally high ($\geq$ 8.8 mg/dL). Elevated creatinine (serum creatinine) usually reflects reduced kidney function (decreased glomerular filtration rate, GFR). The main mechanisms include reduced blood flow to the kidneys (renal hypoperfusion, e.g., dehydration or heart failure), direct kidney damage (intrinsic kidney injury such as acute tubular necrosis), and blockage of urine flow (post-renal obstruction). Less commonly, creatinine may rise due to increased production (muscle breakdown), but impaired filtration is the dominant cause. In this example, creatinine rises from 8.89 (observed) to 10.39 mg/dL; CVTA predicts 9.34 mg/dL (+0.37 from recent observation and +0.08 from last slope attention branch). For this specific instance, we report CVTA's top-8 attended variables in \cref{tab:clinical_case_attention}.

The attended variables cluster around physiologic processes closely related to creatinine: renal perfusion and cardiac function (TroponinT), renal filtration and output (Urine, BUN, Creatinine), acid-base regulation (HCO3), and multi-organ or systemic illness burden (Bilirubin, ALP).

\section{Visualizing Learned Embeddings}
To qualitatively assess the structure of the learned variable representations, we visualize the model’s variable embedding matrix using Uniform Manifold Approximation and Projection (UMAP). After training, we load a pretrained CVTA model checkpoint and extract the learned variable embeddings. The extracted embedding matrix is provided directly as input to the UMAP. We apply UMAP with a neighborhood size of $n_{\text{neighbors}} = 5$ and a minimum distance parameter of $\text{min\_dist} = 0.001$, which emphasizes local neighborhood structure and highlights fine-grained relationships between variables. Dimensionality reduction is performed in an unsupervised manner, generating a two-dimensional embedding used for visualization. To further aid interpretation of the low-dimensional embedding, we apply Density-Based Spatial Clustering of Applications with Noise (DBSCAN) to the two-dimensional UMAP projection. DBSCAN is performed directly on the UMAP coordinates using a neighborhood radius of $\varepsilon = 0.45$ and a minimum cluster size of six points. This procedure identifies groups of variables that are densely packed in the UMAP space, while assigning isolated points to a diffuse (noise) category. Clusters are used solely for visualization and annotation purposes. Each variable is labeled with its corresponding cluster identifier when assigned to a non-noise cluster.

\begin{figure}[t]
    \centering
    \includegraphics[width=0.5\linewidth]{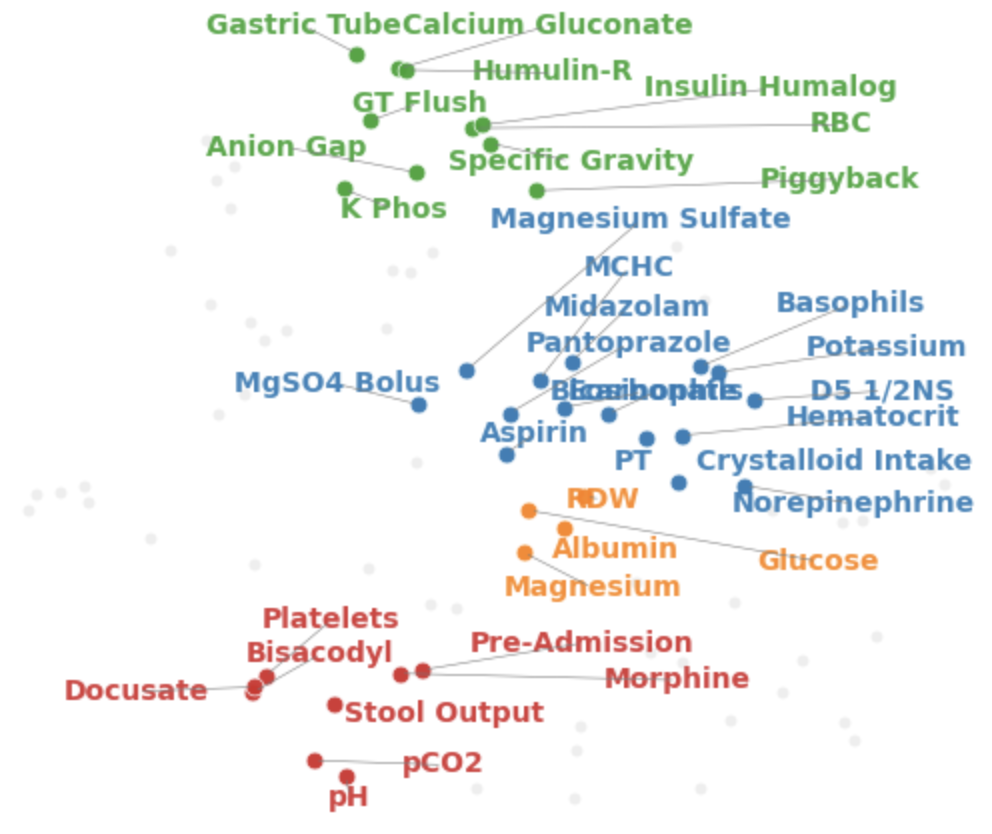}
    \caption{Visualizing the learned UMAP projections pertaining to the variable embeddings in MIMIC-III dataset}
    \label{fig:mimic_iii_emb}
\end{figure}

\begin{figure}[H]
    \centering
    \includegraphics[width=0.6\linewidth]{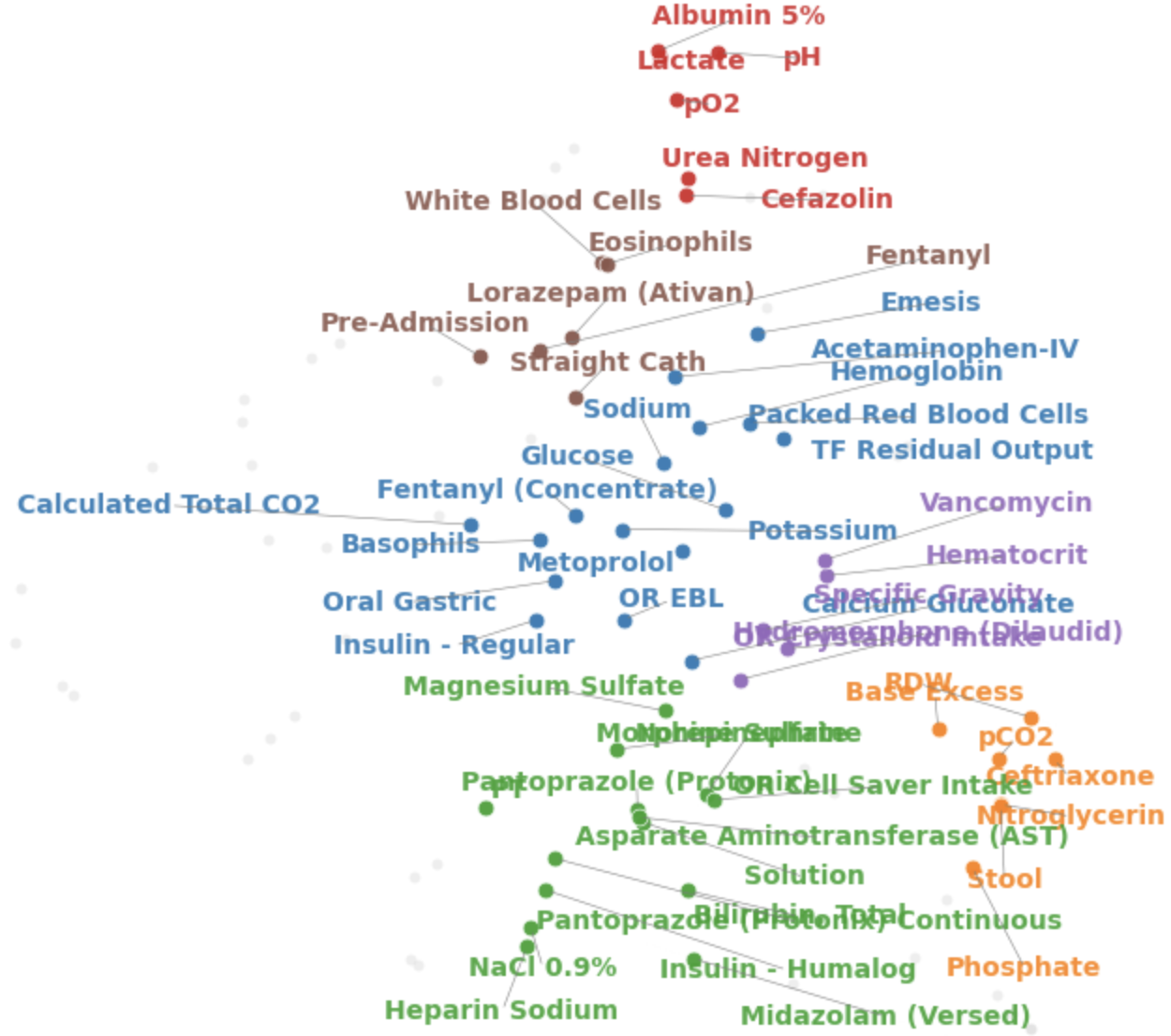}
    \caption{Visualizing the learned UMAP projections pertaining to the variable embeddings in MIMIC-IV dataset}
    \label{fig:mimic_iv_emb}
\end{figure}
\textbf{MIMIC-IV.} The most prominent cluster groups variables related to acute metabolic and hemodynamic management in sedated, mechanically ventilated patients. This cluster integrates laboratory values (Lactate, pH, Hemoglobin), continuous medications (Fentanyl, Midazolam, Insulin), resuscitative products (Packed Red Blood Cells, Albumin 5\%), and invasive monitoring outputs (OR EBL, TF Residual). Secondary clusters appear to group renal-electrolyte management (e.g., Creatinine, Phosphate, Furosemide) and coagulation-related parameters (e.g., Platelets, PTT).

\section{Attention Weight Faithfulness}
\label{section:attn_weights_faithfulness}
\citet{AttNotExp} argue that attention weights may not faithfully reflect what drives a model’s predictions, so they may not be reliable explanations. Following the method described in \citet{SerranoSmith}, to test faithfulness in CVTA, we directly perturb the attention weights and measure the resulting change in prediction error. For each valid target, we mask the selected variables by setting their attention weights to zero and renormalizing the remaining weights so that they form a valid convex combination. We apply this intervention separately to CVTA’s recent observation and local slope attention branches under two settings: Masked Top-3, which removes the three most attended variables; Masked Bottom-3, which removes the three least attended variables. We assess reliability with a paired prediction-level test: for each valid target in the heldout set, we compute the squared-error difference between the perturbed and baseline predictions (i.e. predictions with no attention weight perturbation), and then estimate the mean change in MSE with a 95\% paired bootstrap confidence interval. This tests whether masking the selected variables changes error on the same prediction targets.

For CVTA model trained on Physionet, when the top-3 attended variables are masked, the paired change in MSE is \bm{$+0.0417$}, with a 95\% confidence interval of $[0.0415, 0.0420]$. This shows that removing the most attended variables causes a substantial and highly consistent degradation in the predictive accuracy. In contrast, when the bottom-3 attended variables are masked, the paired change in MSE is $\bm{-1.12\times 10^{-5}}$, with a 95\% confidence interval of $[-1.19\times10^{-5}, -1.05\times10^{-5}]$. This near-zero effect indicates that the least attended variables contribute negligibly to the model’s predictions. Overall, these results provide evidence that the learned attention weights are indeed faithful. 

For additional reference, we provide Baseline MSE (unperturbed), and perturbed MSEs over all prediction targets aggregated by variable type for 20 representative variables.

\begin{table}[H]
\centering
\small
\caption{Effect of masking top- and bottom-attended variables.}
\label{tab:masked_attention_mse}
\begin{tabular}{>{\bfseries\boldmath}l S[table-format=2.2] S[table-format=3.2] S[table-format=2.2]}
\toprule
\textbf{Variable} &
{\makecell{\textbf{Baseline}\\\textbf{MSE} $(\times 10^{-3})$}} &
{\makecell{\textbf{Masked Top-3}\\\textbf{MSE} $(\times 10^{-3})$}} &
{\makecell{\textbf{Masked Bottom-3}\\\textbf{MSE} $(\times 10^{-3})$}} \\
\midrule
Albumin     & 19.67 & 74.42  & 19.93 \\
ALT         & 2.48  & 3.55   & 2.48  \\
Bilirubin   & 7.87  & 13.11  & 7.87  \\
BUN         & 6.76  & 12.61  & 6.75  \\
Creatinine  & 2.02  & 9.06   & 2.02  \\
$\mathrm{FiO}_2$ & 24.45 & 86.91  & 24.44 \\
GCS         & 17.20 & 121.65 & 17.20 \\
Glucose     & 6.41  & 10.47  & 6.41  \\
$\mathrm{HCO}_3$ & 5.36  & 36.36  & 5.36  \\
HR          & 1.27  & 3.45   & 1.27  \\
Lactate     & 4.66  & 8.58   & 4.68  \\
Na          & 1.02  & 47.92  & 1.02  \\
NIMAP       & 4.07  & 145.85 & 4.06  \\
$\mathrm{PaCO}_2$ & 6.60  & 103.55 & 6.61  \\
pH          & 0.37  & 5.05   & 0.37  \\
RespRate    & 1.71  & 60.32  & 1.71  \\
$\mathrm{SaO}_2$ & 1.52  & 199.62 & 1.51  \\
SysABP      & 3.83  & 28.13  & 3.83  \\
Troponin I  & 63.79 & 75.94  & 64.02 \\
Urine       & 0.27  & 9.49   & 0.27  \\
\bottomrule
\end{tabular}
\end{table}

\section{Computing Infrastructure}
\label{section:compute}

We used one server machine to deploy the experiments reported in the paper. This machine is equipped with 100GB memory, one NVIDIA L40S GPU, with Intel Xeon Platinum 8462Y+ @ 2.80 GHz processor and 16 CPU cores.



\end{document}